\documentclass{article}

\PassOptionsToPackage{numbers, square}{natbib}

\usepackage[preprint]{neurips_2025}

\usepackage[utf8]{inputenc} 
\usepackage[T1]{fontenc}    
\usepackage{hyperref}       
\usepackage{url}            
\usepackage{booktabs}       
\usepackage{amsfonts}       
\usepackage{nicefrac}       
\usepackage{microtype}      
\usepackage{xcolor}         
\usepackage{amsmath}
\usepackage{amssymb}
\usepackage{graphicx}
\usepackage{subcaption}
\usepackage{algorithm}    
\usepackage{algorithmic}
\usepackage{multirow} 
\usepackage{amsthm}

\newcommand{\mc}[1]{\mathcal{#1}}
\newcommand{\mbf}[1]{\mathbf{#1}}

\newcommand{\E}{\mathbb{E}}
\newcommand{\R}{\mathbb{R}}

\newcommand{\1}{\mathbf{1}}
\newtheorem{theorem}{Theorem}[section]

\newtheorem{lemma}[theorem]{Lemma}
\newtheorem{corollary}[theorem]{Corollary}

\newtheorem{assumption}[theorem]{Assumption}
\newtheorem{remark}[theorem]{Remark}

\title{Accelerating Decentralized Optimization via Overlapping Local Steps}

\author{%
  Yijie Zhou \\
  School of Data Science\\
  The Chinese University of Hong Kong, Shenzhen\\
  \texttt{yijiezhou@link.cuhk.edu.cn} \\
  \And
  Shi Pu \\
  School of Data Science\\
  The Chinese University of Hong Kong, Shenzhen\\
  \texttt{shipu@cuhk.edu.cn} \\
}

\begin{document}

\maketitle

\begin{abstract}
  Decentralized optimization has emerged as a critical paradigm for distributed learning, enabling scalable training while preserving data privacy through peer-to-peer collaboration. However, existing methods often suffer from communication bottlenecks due to frequent synchronization between nodes. We present Overlapping Local Decentralized SGD (OLDSGD), a novel approach to accelerate decentralized training by computation-communication overlapping, significantly reducing network idle time. With a deliberately designed update, OLDSGD preserves the same average update as Local SGD while avoiding communication-induced stalls. Theoretically, we establish non-asymptotic convergence rates for smooth non-convex objectives, showing that OLDSGD retains the same iteration complexity as standard Local Decentralized SGD while improving per-iteration runtime. Empirical results demonstrate OLDSGD's consistent improvements in wall-clock time convergence under different levels of communication delays. With minimal modifications to existing frameworks, OLDSGD offers a practical solution for faster decentralized learning without sacrificing theoretical guarantees.
\end{abstract}

\section{Introduction}
\label{sec:intro}
In the era of large-scale machine learning, the growing size of training datasets and their frequent distribution across multiple locations have made distributed training a necessity. Traditional centralized approaches—where data is aggregated into a single location—often face critical bottlenecks, including memory constraints, communication overhead, and privacy concerns \cite{verbraeken2020survey}. Distributed training methods, such as federated learning \cite{mcmahan2017communication} and data-parallel training \cite{li2020pytorch} mitigate there issues by enabling collaborative model updates across nodes while keeping data localized. However, these approaches typically rely on a central coordinator (e.g., a parameter server) to synchronize updates, which can become a single point of failure or a communication bottleneck \cite{kairouz2021advances}. 
Decentralized optimization (DO) addresses these limitations by enabling peer-to-peer collaboration without a central coordinator, leveraging local computations and neighbor-to-neighbor communication \cite{nedic2009distributed}. The goal of DO is to optimize a sum of local objectives:
\begin{equation}
\label{eq:DO}
    \min_{x\in \R^d} f(x) = \sum_{i=1}^n f_i(x),
\end{equation}
where $n$ is the number of agents and $d$ is the dimension of the problem, and $f_i(x)\triangleq \E_{\xi\sim \mc{D}_i} F(x;\xi)$ represents the local loss for agent $i$ with data distribution $\mc{D}_i$. Agents solves (\ref{eq:DO}) collaboratively by communicating only with their neighbors in set $\mc{N}_i$.

A key challenge in distributed training is communication overhead, as synchronization causes idle time due to waiting \cite{ben2019demystifying}. Techniques like local updates \cite{stich2018local} and communication-computation overlapping \cite{wang2018cooperative} reduce synchronization frequency and mask communication delays by using stale information (e.g., gradients) for updates. However, staleness can degrade convergence, especially for non-convex objectives, as outdated gradients may misalign with the current optimization direction \cite{arjevani2020tight}.

In decentralized training, communication-computation overlapping remains underexplored. While \cite{wang2024promise} mentions that Decentralized SGD (DSGD) with the combine-then-adapt (CTA) scheme can support overlapping of communication and computation, this idea is not explored in depth, and naive overlapping only brings limited speedup when communication is slow. Building on the potential of DSGD-CTA for overlapping, we propose Overlapping Local DSGD (OLDSGD), the first decentralized training algorithm that combines local updates with communication-computation overlapping. By carefully modifying the update rule, OLDSGD fully masks communication delays without altering the average update, mitigating staleness issues prevalent in distributed training \cite{zhou2021communication}. This enables faster, more robust decentralized training.

\subsection{Related work}
OLDSGD is the first decentralized algorithm to integrate communication-computation overlapping with local updates, reducing communication delays while maintaining convergence guarantee. To contextualize its novelty, we first review distributed training methods that employ similar techniques but rely on centralized coordination. Then we examine decentralized local methods that do not incorporate overlapping, highlighting the unique contribution of OLDSGD.

\subsubsection{Distributed overlapping local methods}
Prior work on overlapping local methods, including  Overlap-Local-SGD \cite{wang2020overlap}, CoCoD-SGD \cite{shen2019faster}, and Co2 \cite{sun2024co2}, fundamentally relies on globally averaged model updates. This architectural choice inherently restricts their applicability to arbitrary network topologies. In Overlap-Local-SGD, agents perform local updates in parallel while periodically communicating to compute a stale global average (referred to as the anchor point). This mechanism enables overlapping computation and communication, with local models subsequently being adjusted toward the anchor to maintain consensus. Similarly, CoCoD-SGD permits nodes to compute with stale gradients while exchanging updates, resembling OLDSGD in fully-connected topologies. Co2 takes a different approach by leveraging asynchronous communication patterns to reduce training time, though this is achieved within centralized system architectures. In contrast, OLDSGD operates with the average model over its neighbors, making it uniquely suitable for fully decentralized settings with arbitrary network structures.


Co2 and most distributed overlapping methods \cite{zhu2021delayed, chen2023workie, kale2025eager} introduce bias in their updates due to the use of stale gradients. These approaches employ various compensation techniques to mitigate the performance degradation caused by this bias: Co2 utilizes gradient clipping and staleness penalties, Delayed-SGD \cite{zhu2021delayed} implements adaptive gradient weighting, and Workie-Talkie \cite{chen2023workie} incorporates error compensation with staleness-aware scheduling. In contrast, OLDSGD employs a intrinsically different approach through its combine-then-adapt (CTA) scheme in decentralized settings. This design inherently achieves unbiased average model updates, thereby circumventing the need for additional correction mechanisms required by other methods.

\subsubsection{Decentralized local methods}
Local DSGD (LDSGD) \cite{koloskova2020unified} extends DSGD by allowing multiple local steps before communication, as illustrated below:
\begin{equation}
    \label{eq:LDSGD}
    x_i^t = \begin{cases}
        x_{i}^{t-1} - \alpha g^{t-1}_i,\text{ if } t \text{ mod } \tau \ne 0,\\
        \sum_{j \in \mc{N}_i} w_{ij}(x_j^{t-1} - \alpha g_i^{t-1}), \text{else},
    \end{cases}
\end{equation}
where $\alpha$ is the learning rate, $w_{ij}$ is the mixing weight, $\tau$ is the number of local steps, and $g_i^t\triangleq \nabla F(x_i^t,\xi_i^t)$ is the stochastic gradient. 

Local Gradient Tracking methods, such as KGT \cite{liu2024decentralized} and LUGT \cite{nguyen2023performance}, build upon the standard Gradient Tracking framework \cite{pu2021distributed} by incorporating local updates to reduce communication frequency. These methods employ correction terms to track the global gradient direction across workers, trying to mitigate the effect of data heterogeneity. Similarly, Local Exact Diffusion (LED) \cite{alghunaim2024local} extends the Exact Diffusion algorithm \cite{yuan2018exact} - which uses a modified consensus step to eliminate gradient correction bias - by allowing multiple local steps between communications. However, these methods do not investigate communication-computation overlap, while OLDSGD fully masks communication delays and maintain convergence guarantee.


\subsection{Main contributions}
\label{sec:contribution}
This paper presents Overlapping Local DSGD (OLDSGD), a novel decentralized training algorithm addressing communication bottlenecks in large-scale machine learning. Our main contributions are:

\begin{itemize}
    \item \textbf{Novel Algorithm}: We propose OLDSGD, the first decentralized training method to integrate local updates with communication-computation overlapping, fully masking communication delays without incurring overhead.
    \item \textbf{Theoretical Guarantees}: We prove that OLDSGD converges at the same rate as Local DSGD \cite{koloskova2020unified} under non-convex settings, maintaining theoretical validity despite overlapping communication.
    \item \textbf{Practical Impact}: We demonstrate OLDSGD's efficiency across multiple vision and language tasks, including VGG11, ResNet18 pretraining and GPT2 finetuning. OLDSGD constantly outperforms existing decentralized methods with significant speedups and beats Local SGD in certain scenarios.
\end{itemize}

These contributions enable faster, scalable decentralized training, overcoming limitations of prior methods reliant on centralized coordination or heavy communication overhead.

\section{Algorithm design}
\label{sec:alg_design}
Overlapping Local DSGD (OLDSGD) enables efficient decentralized training by integrating local updates with communication-computation overlapping. Specifically, each agent performs $\tau$ local gradient steps while exchanging local models with neighbors in $\mc{N}_i$. Upon receiving neighbors' models, the agents perform mixing using a doubly stochastic mixing matrix $W = [w_{ij}] \in \R^{n \times n}$ and applies a gradient update. Algorithm \ref{alg:OLDSGD} describes the implementation of OLDSGD.

\begin{algorithm}
\caption{Overlapping Local Decentralized Stochastic Gradient Descent (OLDSGD)}
\label{alg:OLDSGD}
\begin{algorithmic}[1]
\REQUIRE Number of agents \( N \), total iterations \( T \), step size \( \alpha \), local steps \( \tau \), mixing weights \( w_{ij} \), initial models \( x_i^0 \) for all agents \( i \).
\STATE \textbf{Initialize}: Each agent \( i \) starts sending the initial model \( x_i^0 \) to all neighbors in $\mc{N}_i$.
\FOR{\( t = 1 \) to \( T \)}
    \FOR{each agent \( i = 1 \) to \( N \) in parallel}
        \STATE Fetch data batch $\xi_i^{t-1}$ from $\mc{D}_i$.
        \STATE Compute local stochastic gradient: \( g_i^{t-1} = \nabla F(x_i^{t-1}, \xi_i^{t-1}) \).
        \IF{\(t\text{ mod } \tau \neq 0 \)}
            \STATE Perform local update: \( x_i^t = x_i^{t-1} - \alpha g_i^{t-1} \).
        \ELSE
            \STATE Receive \( x_j^{t-\tau} \) from all neighbors in $\mc{N}_i$.
            \STATE Perform consensus update with delayed gradients: 
            \STATE \( x_i^t = \sum_{j \in \mathcal{N}_i} w_{ij} x_j^{t-\tau} - \alpha \sum_{k=t-\tau}^{t-1} g_i^k \).
            \STATE Send  \( x_i^t \) to all neighbors in $\mc{N}_i$.
        \ENDIF
    \ENDFOR
\ENDFOR
\RETURN Final models \( x_i^T \) for all agents \( i \).
\end{algorithmic}
\end{algorithm}

Figure \ref{fig:schematic_OLDSGD} illustrates the communication and computation timelines of DSGD, Local DSGD (LDSGD), and Overlapping Local DSGD (OLDSGD) over a communication round with $\tau = 5$ local steps. In DSGD, agents wait for neighbor model exchanges at each iteration, incurring significant communication overhead. LDSGD reduces this by communicating once per round, shortening the waiting time. OLDSGD, leveraging communication-computation overlapping, eliminates waiting entirely \footnote{Communication is fully masked as long as communication is not over $\tau$ times slower than gradient computation.} by using stale models and the CTA scheme, enabling up to twice as many gradient updates per round. 

\begin{figure}[h] 
    \centering
    \begin{subfigure}{0.36\textwidth} 
        \includegraphics[width=\textwidth]{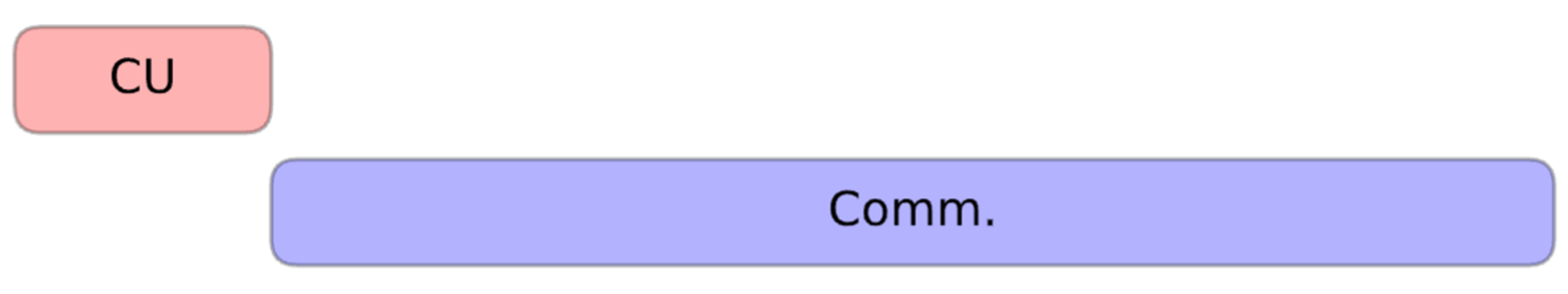}
        \subcaption{DSGD}
        \label{fig:image1}
    \end{subfigure}
    \\
    \begin{subfigure}{0.6\textwidth}
        \includegraphics[width=\textwidth]{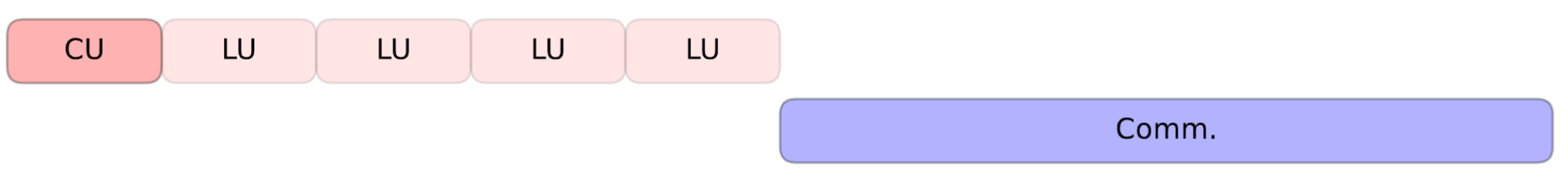}
        \subcaption{LDSGD}
        \label{fig:image2}
    \end{subfigure}
    \\
    \begin{subfigure}{0.6\textwidth}
        \includegraphics[width=\textwidth]{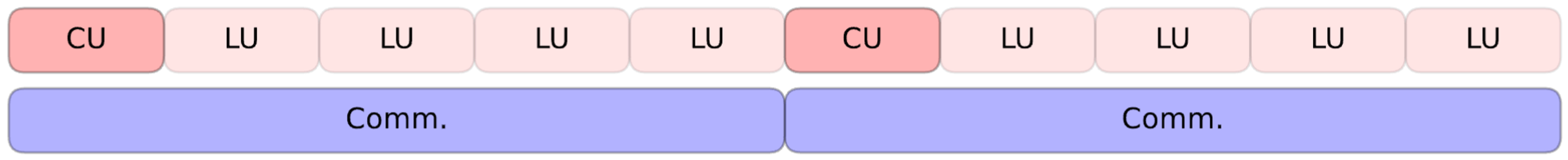}
        \subcaption{OLDSGD}
        \label{fig:image3}
    \end{subfigure}
    \caption{Communication-computation schematics of DSGD, LDSGD, and OLDSGD. CU: consensus update, LU: local update, Comm.: communicaiton.}
\label{fig:schematic_OLDSGD}
\end{figure}




The OLDSGD update rule is:
\begin{equation}
    \label{eq:OLDSGD}
    x_i^t = \begin{cases}
        x_i^{t-1} - \alpha g_i^{t-1}, & \text{if } t\text{ mod }\tau \neq 0, \\
        \sum_{j \in \mc{N}_i} w_{ij} x_j^{t-\tau} - \alpha \sum_{k=t-\tau}^{t-1} g_i^k, & \text{otherwise},
    \end{cases}
\end{equation}
where $x_i^t$ is agent $i$'s model at iteration $t$, $\alpha$ is the step size, and $g_i^k$ is the local stochastic gradient.

Compared to Local DSGD (equation (\ref{eq:LDSGD})) \cite{koloskova2020unified}\footnote{OLDSGD is not covered in the framework of \cite{koloskova2020unified} due to the use of stale models.}, OLDSGD’s communication step differs in two ways. First, it adopts the combine-then-adapt (CTA) scheme, using only local gradients $\sum_{k=t-\tau}^{t-1} g_i^k$, eliminating the need for fresh neighbor gradients. Second, it uses stale neighbor models $x_j^{t-\tau}$ instead of $x_j^{t-1}$, enabling communication to overlap with local gradient computations, thus masking communication delays.

A key property of OLDSGD is that its average model, $\bar{x}^t \triangleq \frac{1}{n} \sum_{i=1}^n x_i^t$, follows an SGD-like update:
\begin{equation}
    \label{eq:ave_OLDSGD}
    \bar{x}^t = \bar{x}^{t-1} - \frac{\alpha}{n} \sum_{i=1}^n g_i^{t-1}, \quad \forall t.
\end{equation}
To see why this is the case, for $t\text{ mod } \tau \neq 0$, the relation holds trivially. When $t$ is a multiple of $\tau$, the update becomes:
\begin{align*}
    \bar{x}^t &= \frac{1}{n} \sum_{i=1}^n \left( \sum_{j \in \mc{N}_i} w_{ij} x_j^{t-\tau} - \alpha \sum_{k=t-\tau}^{t-1} g_i^k \right) \\
    &= \bar{x}^{t-\tau} - \frac{\alpha}{n} \sum_{i=1}^n \sum_{k=t-\tau}^{t-1} g_i^k ,
\end{align*}
where the second equality comes from the double stochasticity of $W$ (i.e., $\sum_{j} w_{ij} = 1$, $\sum_{i} w_{ij} = 1$), and \eqref{eq:ave_OLDSGD} follows from the recursive application of local updates.

The SGD-like average update explains why OLDSGD avoids the gradient mismatch issue common in overlapping distributed methods \cite{zhu2021delayed,zhou2021communication}. The CTA scheme ensures compatibility with overlapping, minimizing performance degradation, making OLDSGD particularly effective for decentralized training.

Similar algorithmic design can be naively applied to Local Exact Diffusion \cite{alghunaim2024local} and Local Update Gradient Tracking \cite{nguyen2023performance}. However, their overlapping versions suffer significant convergence degradation, both in terms of convergence speed and the number of permissible local steps. The staleness error appears to have a substantial impact due to their more complex update mechanisms. Furthermore, the update structure of KGT \cite{liu2024decentralized} inherently prevents overlapping implementation. These observations explain why this work focuses exclusively on developing OLDSGD. However, note that full masking occurs only when all agents have homogeneous computation and communication speeds. Otherwise, stalls persist due to synchronization delays. Readers can find a detailed discussion of these comparative aspects in the Appendix.

\section{Convergence Analysis}
\label{sec:conv_res}
\subsection{Assumptions}
\label{sec:assu}
We state the assumptions for the convergence analysis of OLDSGD.

\begin{assumption} \label{as:d_sto}
The mixing matrix $W$ is doubly stochastic, and the corresponding communication graph is undirected and connected.
\end{assumption}

\begin{assumption}\label{as:L_smooth}
Each loss objective $f_i:\R^d \rightarrow \R$ is $L$-smooth, satisfying:
    \begin{equation*}
    \|\nabla f_i(x) - \nabla f_i(y)\| \le L \|x - y\|,\quad \forall x,y\in \R^d.
    \end{equation*}
\end{assumption}

\begin{assumption}\label{as:var}
For each agent $i$, the stochastic gradient variance satisfies:
    \begin{equation*}
    \frac{1}{n} \sum_{i=1}^n \mathbb{E}_{\xi_i} \left\| \nabla F(x_i, \xi_i) - \nabla f_i(x_i) \right\|_2^2 \leq \sigma^2 + \frac{M}{n} \sum_{i=1}^n \| \nabla f_i(x_i) \|_2^2.
    \end{equation*}
\end{assumption}

\begin{assumption}\label{as:data_hete}
The data heterogeneity across agents is bounded:
    \begin{equation*}
    \frac{1}{n}\sum_{i=1}^n \|\nabla f_i(x) - \nabla f(x)\|^2 \le  \zeta^2 + P\|\nabla f(x)\|^2.
    \end{equation*}
\end{assumption}

\begin{assumption}\label{as:lb}
The global objective $f:\R^d \rightarrow f^*$ is lower bounded by a finite $f^*$.
\end{assumption}

\begin{table*}[t]
\centering
\caption{Key notations for OLDSGD analysis.}
\label{tab:notations}
\footnotesize
\begin{tabular}{lp{0.34\textwidth}lp{0.34\textwidth}}
\toprule
\textbf{Notation} & \textbf{Description} & \textbf{Notation} & \textbf{Description} \\
\midrule
$n$ & Number of agents & $\mc{N}_i$ & Neighboring agents of agent $i$ \\
$f_i,f$ & Local and global objective& $x_i^t,\bar{x}^t$ & Local and average model\\
$W = [w_{ij}]$ & Doubly stochastic mixing matrix & $\lambda_2$ & Second largest eigenvalue of $W$ in magnitude\\
$\tau$ & Number of local steps & $\alpha$ & Step size for gradient updates \\
$g_i^k$ & Stochastic gradient at iteration $k$ & $L$ & Smoothness constant  \\
$\sigma^2, M$ & Noise related constants & $\zeta^2,P$ & Data heterogeneity related constants \\
$C, D$ & Constants in Theorem \ref{th:overall_con} & $p$ & $1-\lambda_2^2$, the spectral constant \\
$C_0,...,C_3$ & Temporary constants in \eqref{eq:sch_des} and \eqref{eq:sch_con} &$c$ & The ratio of comm. and comp. times\\
\bottomrule
\end{tabular}
\end{table*}

\subsection{Convergence of OLDSGD}
\label{sec:conv_OLDSGD}
The key notations used in the main text are summarized in Table \ref{tab:notations}. As mentioned in Section \ref{sec:alg_design}, the average update of OLDSGD follows an SGD-like dynamic, enabling a descent lemma for the global objective $f$ with the following form:
\begin{equation}
\label{eq:sch_des}
    \E_{t+1} f(\bar x^{t+1}) \le f(\bar x^{t}) - \alpha C_0 \|\nabla f(\bar x^t)\|^2 + \alpha C_1 \sum_{i=1}^n \|x_i^t - \bar x^t\|^2 + \alpha^2C_2,
\end{equation}
where the RHS can be controlled as long as the consensus error is controlled. Our key insight is that OLDSGD preserves a similar consensus error bound as in LDSGD. Intuitively, this occurs because (1) OLDSGD still performs periodic consensus averaging, and (2) the distance between stale and fresh models can be properly bounded. As a result, OLDSGD maintains the same essential consensus property as LDSGD. Specifically, we prove that the cumulative expected consensus error in OLDSGD satisfies the following inequality:
\begin{equation}
\label{eq:sch_con}
    C_1\sum_{t=0}^{T-1} \sum_{i=1}^n\E\|x_i^t - \bar x^t\|^2 \le \frac{C_0}{2}\sum_{t=0}^{T-1} \E\|\nabla f(\bar x^t)\|^2 + \alpha^2C_3.
\end{equation}

Combining \eqref{eq:sch_des} and \eqref{eq:sch_con}, we arrive at the main convergence theorem of OLDSGD under non-convexity and a corollary that considers a specfic step size.

\begin{theorem}
\label{th:overall_con}
Given Assumption \ref{as:d_sto}-\ref{as:lb}, the average model generated by equation \eqref{eq:OLDSGD} satisfies
    \begin{align*}
        \frac{\sum_{t=0}^{T-1} \E \|\nabla f(\bar x^t)\|^2}{T} &\le \frac{8(f^0 - f^*)}{\alpha T} + \frac{8\alpha L}{n}(\frac{\sigma^2}{2} + M\zeta^2)+ \frac{1024 L^2}{n}\frac{D\tau}{p}\alpha^2,
    \end{align*}
    where 
    \begin{equation*}
       \alpha \le \min\{\frac{1}{L}, \frac{n}{4LM(P+1)},\frac{n}{LM},\frac{p}{16L\sqrt{3\tau(2\tau+M)}}, \frac{1}{32L} \sqrt{\frac{pn} {2C\tau}}\}, 
    \end{equation*}
     and
    \begin{equation*}
        C=\frac{12(2\tau+M)n(P+1)}{p}, D=\frac{6((2\tau+M)n\zeta^2+ n\sigma^2)}{p}.
    \end{equation*}
\end{theorem}

\begin{corollary}
\label{coro:1}
Under the conditions in Theorem \ref{th:overall_con}, let $\alpha = \sqrt\frac{n}{T}$, we have 
\begin{align*}
    \frac{\sum_{t=0}^{T-1} \E \|\nabla f(\bar x^t)\|^2}{T} &\le \frac{8(f^0 - f^*)}{\sqrt{nT}} + \frac{8L(\frac{\sigma^2}{2} + M\zeta^2)}{\sqrt{nT}}+ \frac{6144n L^2\tau((2\tau+M))\zeta^2 + \sigma^2)}{p^2T},
\end{align*}
given that
    \begin{align*}
        T \ge \max\{nL^2, \frac{16L^2M^2(P+1)^2}{n}, \frac{L^2M^2}{n}, \frac{768nL^2\tau(2\tau+M)}{p^2}, \frac{24576nL^2\tau(2\tau+M)(P+1)}{p^2}\}.
    \end{align*}
\end{corollary}

\begin{remark}
    As shown in Corollary \ref{coro:1}, OLDSGD achieves a convergence rate of $\mc{O}(\frac{1}{\sqrt{nT}} + \frac{1}{T})$, comparable to previously established result of Local DSGD under non-convexity \cite{koloskova2020unified,li2019communication}. Namely, the overlapping update \eqref{eq:OLDSGD} does not degenerate any theoretical guarantees, compared to Local DSGD.
\end{remark}

\section{Experiments}
\label{sec:exp}

We evaluate Overlapping Local Decentralized Stochastic Gradient Descent (OLDSGD) across vision and language tasks, comparing it against state-of-the-art decentralized methods: Local DSGD (LDSGD) \cite{koloskova2020unified}, KGT \cite{liu2024decentralized}, LUGT \cite{nguyen2023performance}, LED\footnote{Results of LED is not covered in the main text since it diverges constantly; please refer to the Appendix for details.} \cite{alghunaim2024local}, and Local SGD (LSGD) \cite{stich2018local} with ring-allreduce communication. Experiments are conducted on a server with 8 NVIDIA RTX3090 GPUs (24GB memory each) using PyTorch 2.0.1. Implementations adhere to original algorithm descriptions without additional accelerations. We test homogeneous and heterogeneous data distributions and model communication delays to assess runtime convergence and scalability. 

\subsection{Experimental Setup, Tasks, and Hyperparameters}
\label{sec:exp_setup}
\textbf{Setup.} Experiments use an undirected ring topology unless stated otherwise. To compare convergence w.r.t. time, we model the delays incurred by communication and computation. For computation, we assume each agent take 1 unit of time to compute a local (stochastic) gradient per step, ignoring other operations (e.g., model updates via addition/subtraction), as they are negligible in comparison.  For communication, each agent incurs a delay of \( c \) units of time per consensus step to transmit its model parameters to neighbors. Table~\ref{tab:runtime} presents the per-iteration runtime of each algorithm as a function of \( \tau \) (number of local steps) and \( c \), with the last column showing the normalized runtime when \( \tau = c \), relative to OLDSGD. OLDSGD’s runtime benefits from communication-computation overlap, improving by over 50\% compared to existing methods. 

\begin{table}[h]
\centering
\caption{Per-iteration runtime for different algorithms, with normalized values when \( \tau = c \), relative to OLDSGD. LSGD adopts ring-allreduce for communication.}
\label{tab:runtime}
\begin{tabular}{ccc}
\toprule
Algorithm & Per-Iteration Runtime & Normalized Value (\( \tau = c \)) \\
\midrule
OLDSGD & \( \max\{\tau, c\} \) & 1 \\
LDSGD & \( \tau + c \) & 2 \\
KGT & \( \max\{\tau, c\} + c \) & 2 \\
LED & \( \tau + c \) & 2 \\
LUGT & \( \tau + 2c \) & 3 \\
LSGD & \( \tau + \frac{2(n-1)}{n}c \) & \( 1 + \frac{2(n-1)}{n} \) \\
\bottomrule
\end{tabular}
\end{table}

\textbf{Tasks and Datasets.} We evaluate three tasks:
\begin{enumerate}
    \item \textbf{VGG11} \cite{simonyan2014very} Training on CIFAR-10 (50,000 training, 10,000 test images), with 9 agents.
    \item \textbf{ResNet18} \cite{he2016deep} Training on CIFAR-10, with 9 agents.
    \item \textbf{GPT2-Small} \cite{radford2019language} Finetuning on WikiText-2 \cite{merity2016pointer} (2M tokens), with 8 agents (homogeneous only).
\end{enumerate}

For VGG11 and ResNet18 training, we evaluated both homogeneous and heterogeneous data distributions. In the homogeneous setting, we evenly distributed the dataset across agents to ensure balanced label representation. For the heterogeneous setting, we partitioned the dataset by class labels to introduce data skew. Specifically, 70\% of each agent's data is assigned based on specific labels, with the remaining 30\% uniformly sampled across all labels. For GPT2 finetuning, we only evaluated the homogeneous data distribution due to computation constraints.

\textbf{Hyperparameters.} We configure hyperparameters to ensure a fair comparison. We select a fixed learning rate of \( \alpha = 0.01 \) for all tasks, determined to be suitable based on preliminary experiments, and apply no learning rate scheduling to maintain consistency. For GPT2 finetuning, we employ gradient clipping with a maximum gradient norm of 1 to stabilize training of the large language model. Local batch sizes were set as follows: 32 for logistic regression, 8 for VGG11 and ResNet18, and 2 for GPT2 finetuning.

\begin{figure*}[htbp]
\centering
\hspace{-0.2cm}
\begin{subfigure}[t]{0.31\textwidth}
\centering
\includegraphics[width=\textwidth]{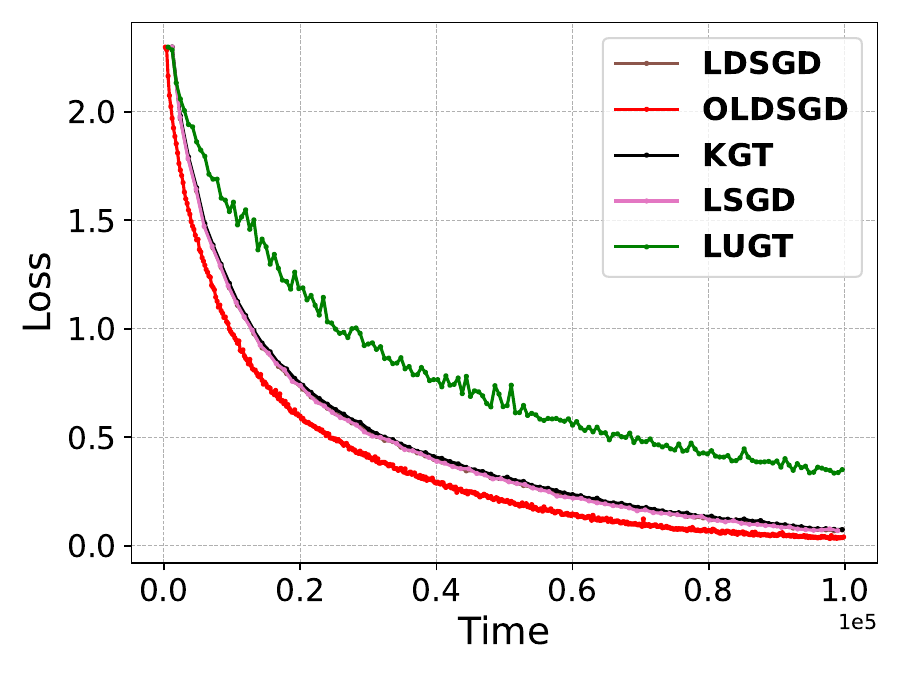}
\caption{Homo.-VGG11}
\label{fig:LOG_stra_loss}
\end{subfigure}
\begin{subfigure}[t]{0.31\textwidth}
\centering
\includegraphics[width=\textwidth]{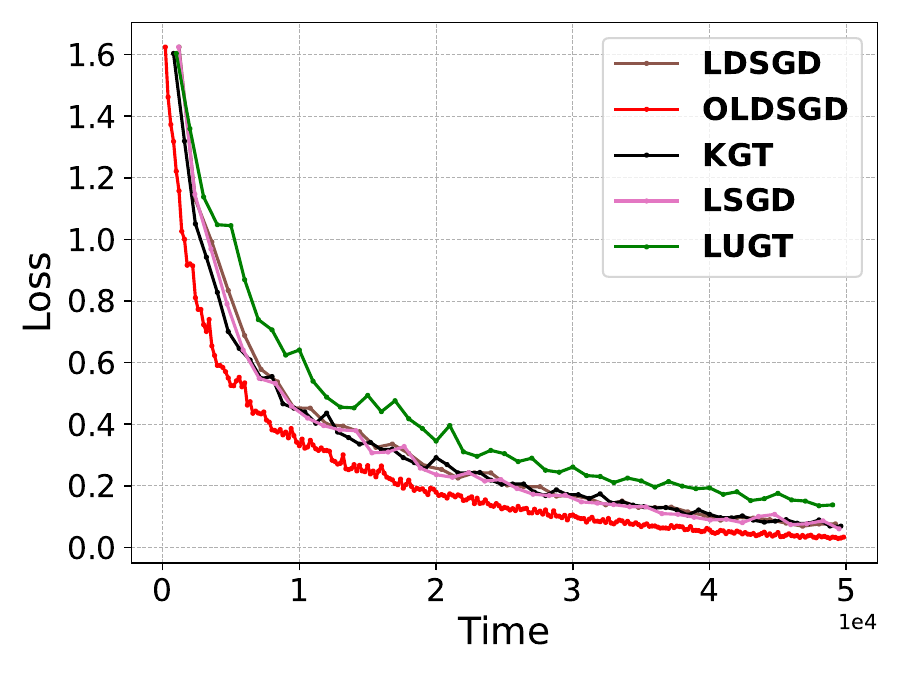}
\caption{Homo.-ResNet18}
\label{fig:LOG_stra_loss}
\end{subfigure}
\begin{subfigure}[t]{0.31\textwidth}
\centering
\includegraphics[width=\textwidth]{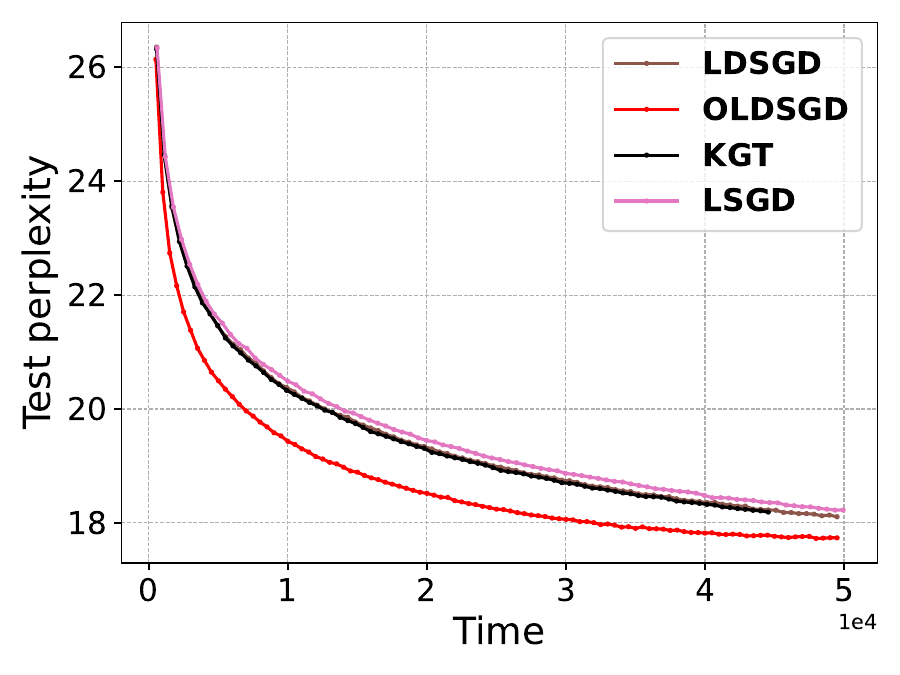}
\caption{Homo.-GPT2}
\label{fig:LOG_stra_loss}
\end{subfigure}
\begin{subfigure}[t]{0.31\textwidth}
\centering
\includegraphics[width=\textwidth]{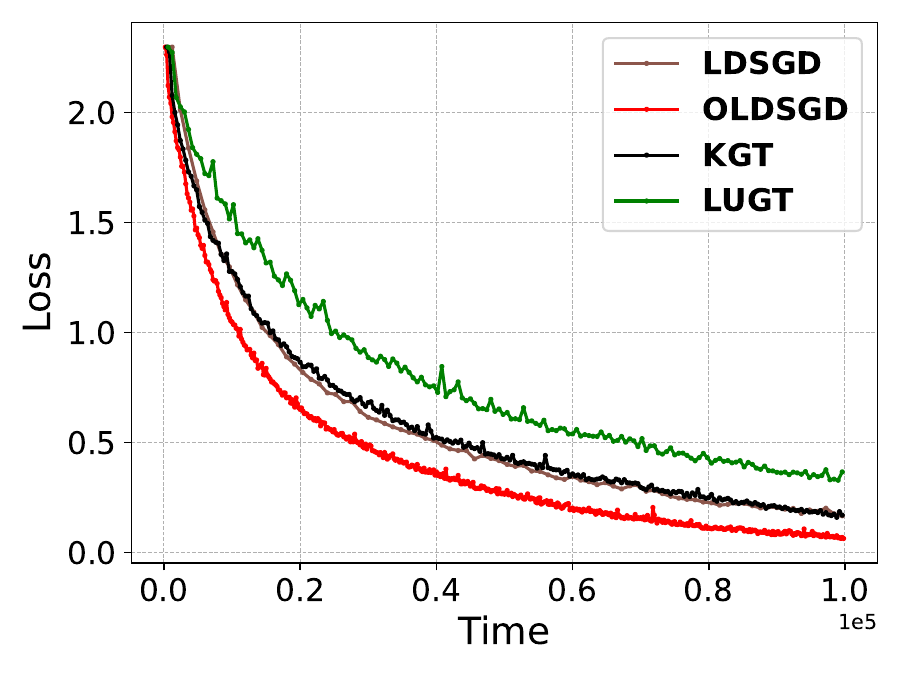} 
\caption{Hete.-VGG11}
\label{fig:LOG_stra_loss}
\end{subfigure}
\begin{subfigure}[t]{0.31\textwidth}
\centering
\includegraphics[width=\textwidth]{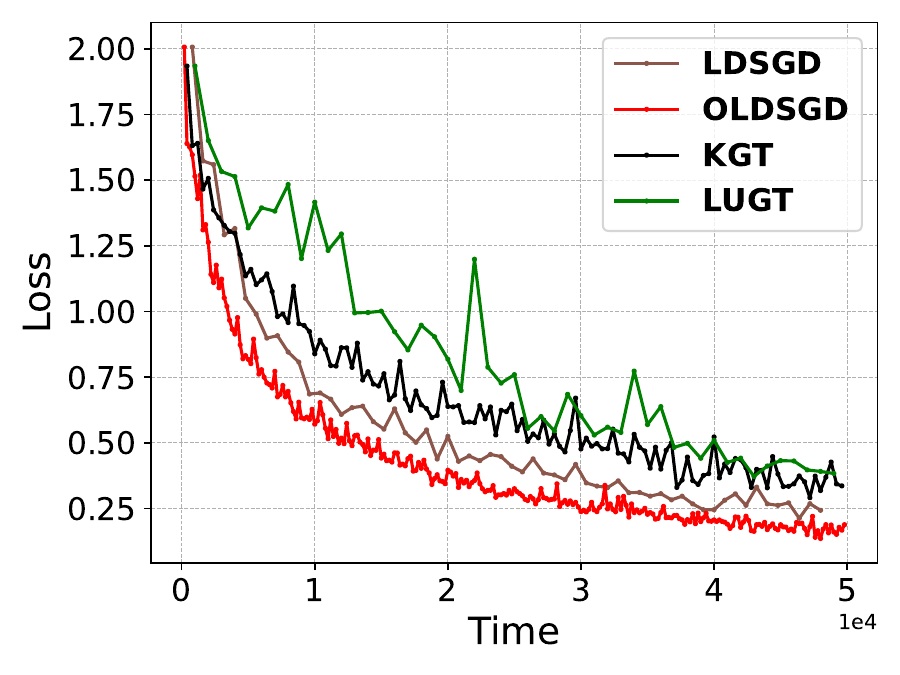}
\caption{Hete.-ResNet18}
\label{fig:LOG_stra_loss}
\end{subfigure}
\caption{Convergence w.r.t. time under different tasks (c=1). The first row presents all homogeneous cases, while the second row presents all heterogeneous cases.}
\label{fig:loss_time_c=1}
\end{figure*}

\subsection{Comparing convergence w.r.t. time}
\label{sec:alg_comp}
We analyze convergence performance, quantified as loss reduction over (simulated) wall-clock time, under communication delays of \( c = 1 \) and \( c = 5 \), where \( c \) denotes the communication time in units relative to a single stochastic gradient computation (set to 1 unit). The \( c = 1 \) scenario assumes communication speed matches computation, while \( c = 5 \) models communication being five times slower, reflecting common real-world decentralized training environments with network latency. To ensure a fair comparison, we evaluated a broad range of local steps \( \tau \in \{1, 3, 5, 10, 15, 20, 30, 40\} \) for each algorithm, selecting the optimal \( \tau \) that maximizes convergence speed for each method, thereby approximating their best achievable performance under the given conditions.

Fig. \ref{fig:loss_time_c=1} and Fig. \ref{fig:loss_time_c=5} illustrate the convergence performance across various algorithms and tasks. The top row displays results for homogeneous data distributions, while the bottom row corresponds to heterogeneous distributions. For GPT2 finetuning, we exclude LUGT and LED due to their tendency to diverge in simpler tasks and limit the evaluation to the homogeneous setting to manage computational constraints. For GPT2, the vertical axis represents the perplexity of the average model on the test dataset, as computing training set metrics is computationally intensive. For other tasks, the vertical axis reflects the training loss of the average model, denoted as $f(\bar x^k)$ in our analysis.

In all evaluated scenarios, OLDSGD consistently outperforms other algorithms in convergence speed. Notably, despite the theoretical advantages of gradient tracking and exact diffusion methods, they generally exhibit slower convergence (or divergence for LED; see Appendix \ref{appen:OLED_OLUGT}) than vanilla gradient descent-based methods, such as OLDSGD and LDSGD, under the same learning rate. Specifically, LUGT, which naively extends Gradient Tracking with local steps, converges more slowly, particularly for \(\tau > 1\), as its loss curves lag significantly behind OLDSGD in Fig. \ref{fig:loss_time_c=1} and \ref{fig:loss_time_c=5} (see Appendix for further analysis). KGT is more robust due to its update mechanism, achieving convergence rates comparable to LDSGD in most tasks (e.g., VGG11 and GPT2 finetuning) at \(c = 1\), except for ResNet18 under heterogeneous data, where its loss reduction is noticeably slower. At \(c = 5\), KGT’s performance declines further under heterogeneity for VGG11 and ResNet18, as evidenced by its flatter loss curves in Fig. \ref{fig:loss_time_c=5}.

\begin{figure*}[t]
\centering
\hspace{-0.2cm}
\begin{subfigure}[t]{0.31\textwidth}
\centering
\includegraphics[width=\textwidth]{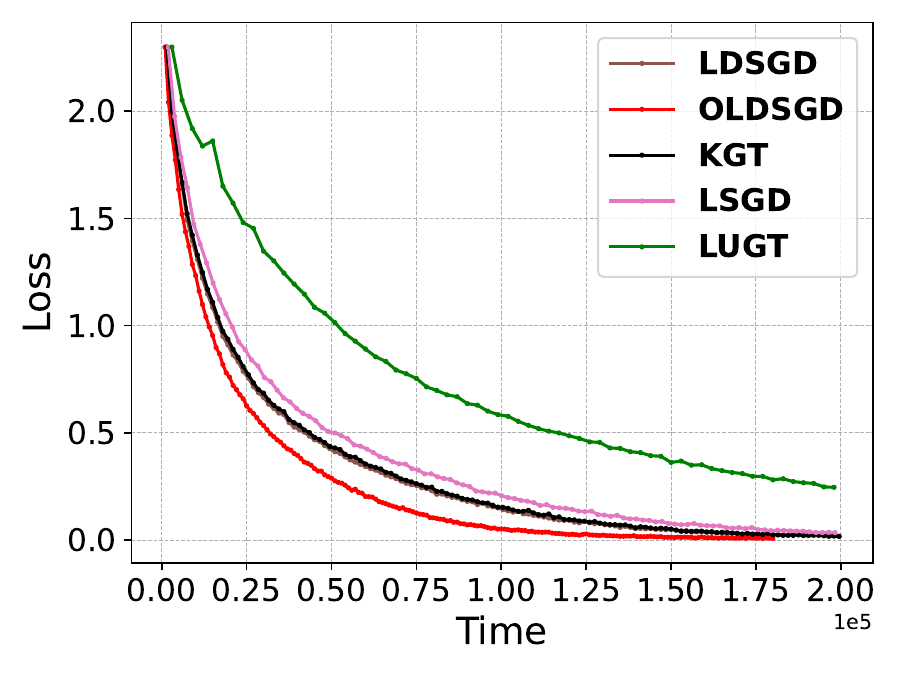}
\caption{Homo.-VGG11}
\label{fig:LOG_stra_loss}
\end{subfigure}
\begin{subfigure}[t]{0.31\textwidth}
\centering
\includegraphics[width=\textwidth]{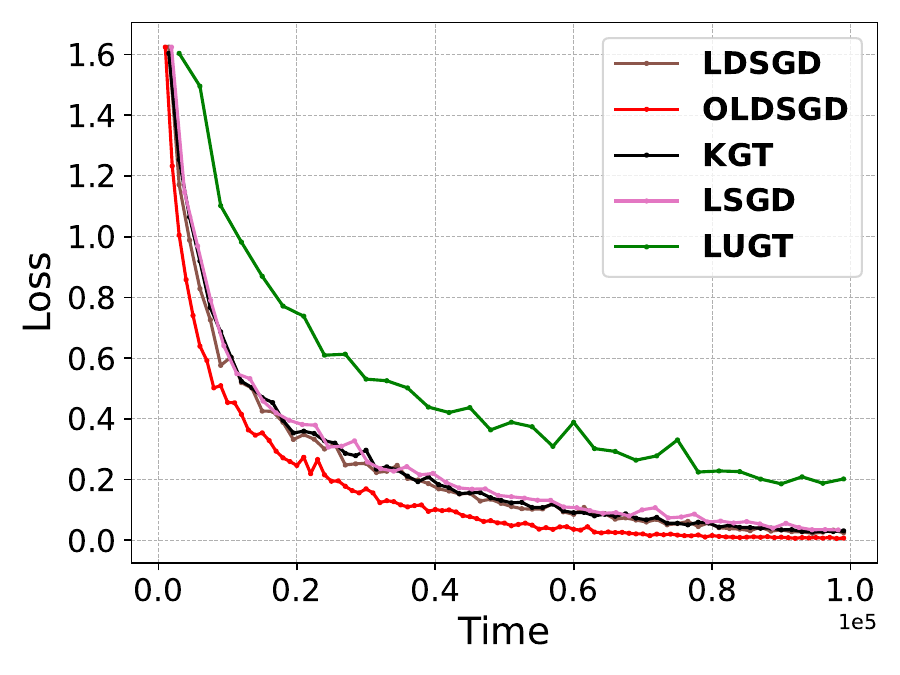}
\caption{Homo.-ResNet18}
\label{fig:LOG_stra_loss}
\end{subfigure}
\begin{subfigure}[t]{0.31\textwidth}
\centering
\includegraphics[width=\textwidth]{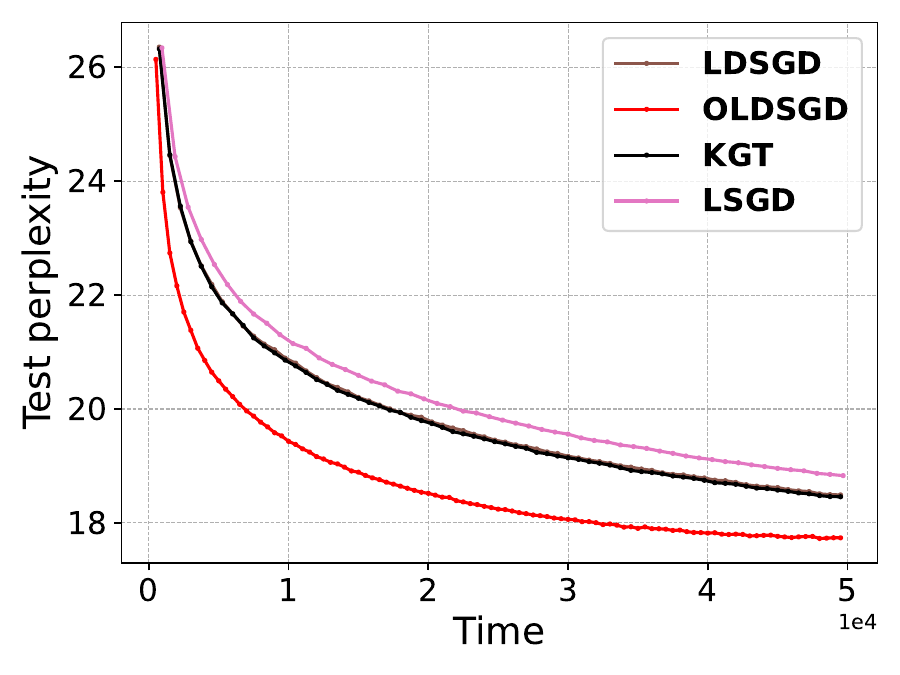}
\caption{Homo.-GPT2}
\label{fig:LOG_stra_loss}
\end{subfigure}
\begin{subfigure}[t]{0.31\textwidth}
\centering
\includegraphics[width=\textwidth]{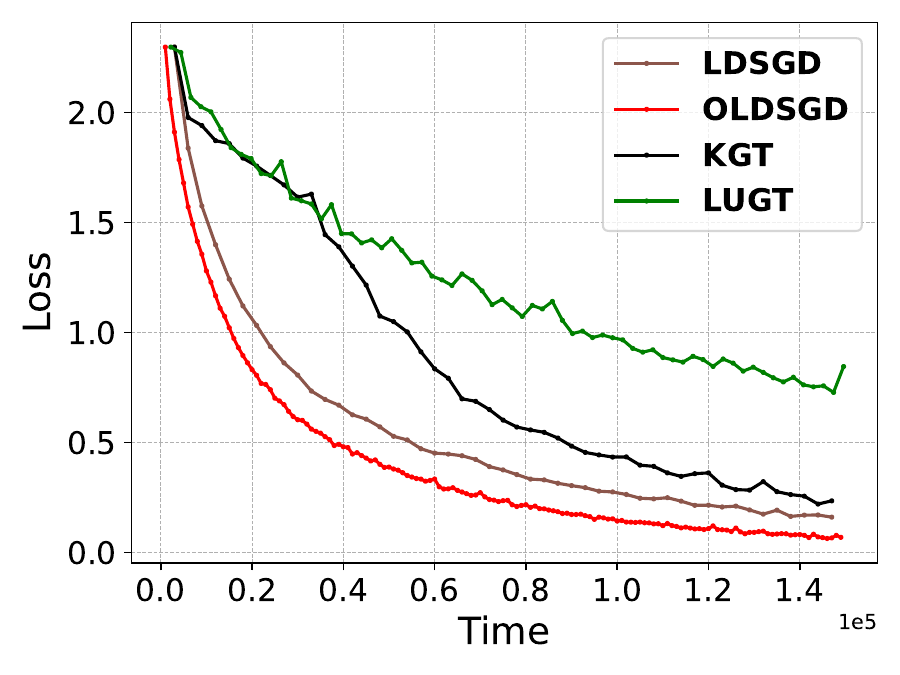} 
\caption{Hete.-VGG11}
\label{fig:LOG_stra_loss}
\end{subfigure}
\begin{subfigure}[t]{0.31\textwidth}
\centering
\includegraphics[width=\textwidth]{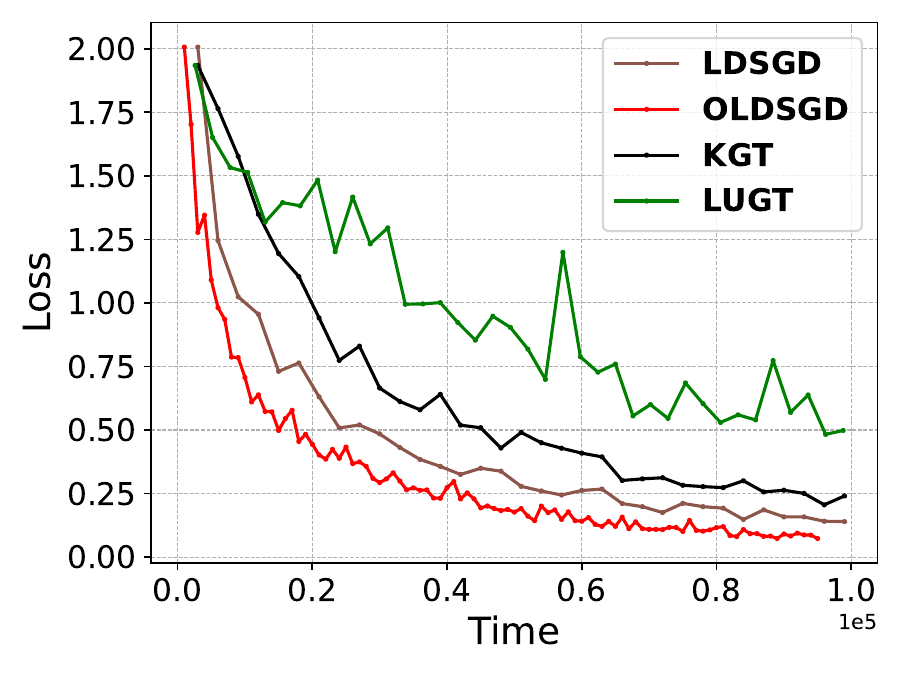} 
\caption{Hete.-ResNet18}
\label{fig:LOG_stra_loss}
\end{subfigure}
\caption{Convergence w.r.t. time under different tasks (c=5). The first row presents all homogeneous cases, the second presents all heterogeneous cases.}
\label{fig:loss_time_c=5}
\end{figure*}

Table \ref{tab:speedup} quantifies OLDSGD’s speedup relative to other algorithms, focusing on the case of homogeneous data distribution (results for heterogeneous settings are deferred to the appendix due to space constraints). The speedup is calculated as the relative time (normalized to OLDSGD’s runtime) required to reach a predefined test accuracy for CIFAR-10 tasks, or perplexity for GPT2 finetuning. The geometric mean of OLDSGD’s speedup over LDSGD exceeds 1.64, reflecting a significant performance gain with minimal algorithmic modifications. Notably, OLDSGD’s efficiency is more pronounced at \( c = 5 \), where communication is five times slower than computation, as evidenced by higher speedup ratios (e.g., 2.65\(\times\) for GPT2 at \( c = 5 \) vs. 1.95\(\times\) at \( c = 1 \)). This trend underscores OLDSGD’s effectiveness in communication-constrained environments, which are prevalent in real-world decentralized systems, such as edge computing or federated learning. Strikingly, OLDSGD achieves significant speedup (2.08\(\times\) and 3.32\(\times\)) over LSGD in GPT2 finetuning, with its perplexity curves in Figure \ref{fig:loss_time_c=5} converging sharply compared to the flatter trajectories of baselines. This remarkable performance highlights OLDSGD’s potential for scaling to modern transformer architectures, where communication bottlenecks are critical.


\begin{table}[htbp]
\centering
\caption{OLDSGD's Speedup Compared to Existing Methods (Higher Speedup is Better)}
\label{tab:speedup}
\begin{tabular}{ccccccccc}
\toprule
 & \multicolumn{2}{c}{\textbf{VGG11}} & \multicolumn{2}{c}{\textbf{ResNet18}} & \multicolumn{2}{c}{\textbf{GPT2}} & \\
\cmidrule(lr){2-3} \cmidrule(lr){4-5} \cmidrule(lr){6-7} 
\textbf{Algorithm} & $c=1$ & $c=5$ & $c=1$ & $c=5$ & $c=1$ & $c=5$ &  {\textbf{GeoMean}} \\
\midrule
LDSGD & 1.23$\times$ & 1.26$\times$ & 1.50$\times$ & 1.62$\times$ & 1.95$\times$ & 2.65$\times$ & 1.64$\times$  \\
KGT & 1.26$\times$ & 1.30$\times$ & 1.42$\times$ & 1.73$\times$ & 1.86$\times$  & 2.54$\times$ & 1.63$\times$ \\
LSGD & 1.34$\times$ & 1.64$\times$ & 1.84$\times$ &2.18$\times$ & 2.08$\times$ & 3.32$\times$ & 1.98$\times$  \\
LUGT & 2.68$\times$ & 3.53$\times$ & 1.98$\times$ & 4.62$\times$ & - & - & 3.05$\times$  \\
\bottomrule
\end{tabular}
\end{table}

\subsection{Scalability}
Previous experiments demonstrate OLDSGD’s robust performance across diverse tasks, but computational constraints restricts the number of agents tested. To evaluate its scalability, we measure speedup relative to the number of agents for VGG11 and ResNet18 training on CIFAR-10, with local steps fixed at \(\tau = 5\). We test agent counts from 2 to 32 in a ring topology and observed significant speedup, as shown in Figure \ref{fig:scalability}. OLDSGD achieves near-linear speedup up to 16 agents, with VGG11 and ResNet18 attaining approximately 14\(\times\) and 13\(\times\) faster convergence, respectively, compared to a single agent. Speedup growth slows slightly beyond 16 agents. This is primarily due to worse graph connectivity in the ring topology, which limits efficient information exchange as the network scales. These results highlight OLDSGD’s scalability and potential for large-scale decentralized training.

\begin{figure*}[htbp]
\centering
\begin{subfigure}[t]{0.45\textwidth}
\centering
\includegraphics[width=\textwidth]{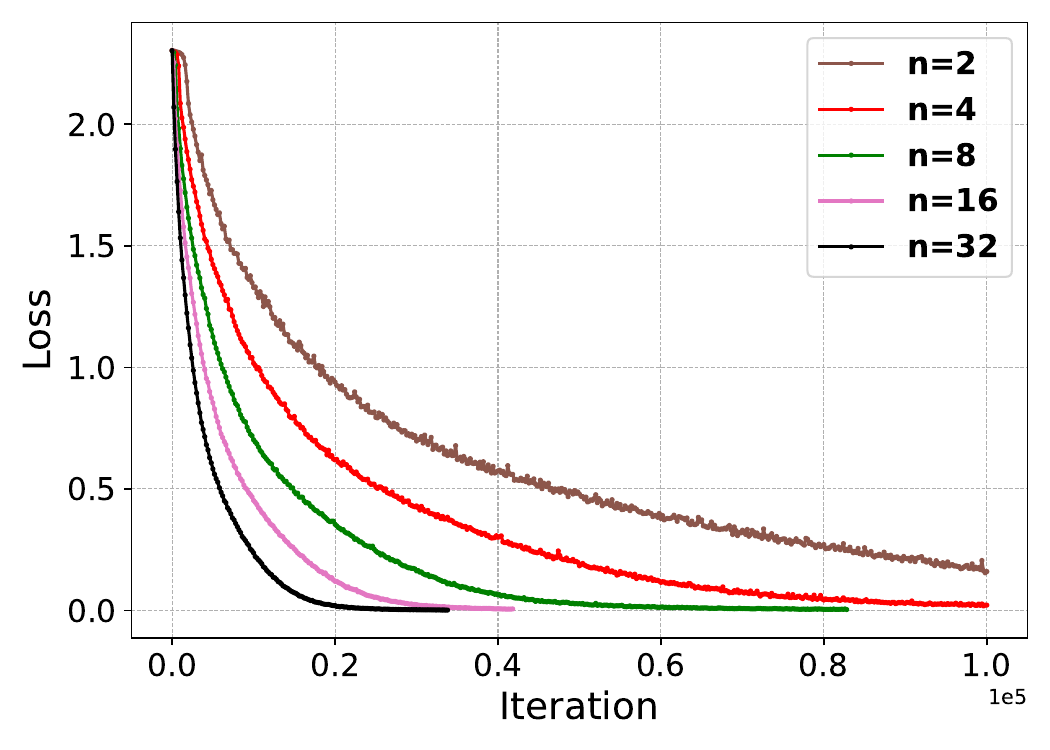}
\caption{VGG11} %
\label{fig:LOG_no_stra_loss}
\end{subfigure}
\hspace{-0.2cm}
\begin{subfigure}[t]{0.45\textwidth}
\centering
\includegraphics[width=\textwidth]{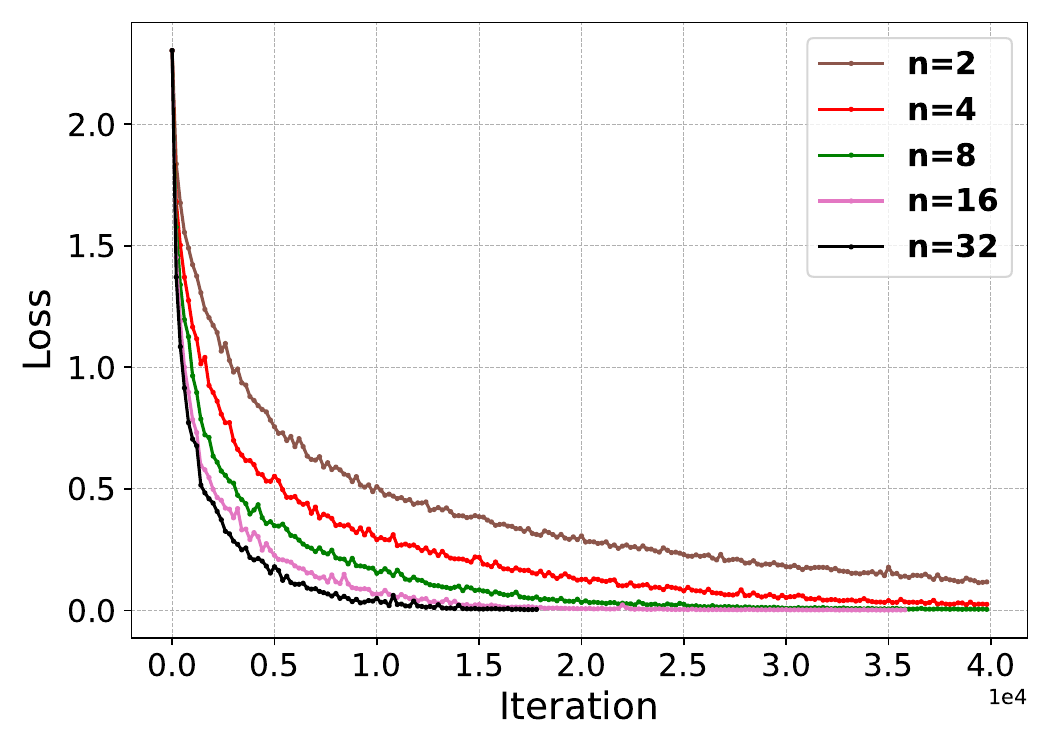}
\caption{ResNet18}
\label{fig:LOG_stra_loss}
\end{subfigure}
\caption{Scalability of OLDSGD on VGG11 (a) and ResNet18 (b) (CIFAR-10, ring topology, $\tau=5$). Near-linear speedup is achieved up to 16 agents (14× for VGG11, 13× for ResNet18), with diminishing returns beyond due to ring topology constraints. Loss curves demonstrate accelerated convergence with increasing parallelism.}

\label{fig:scalability}
\end{figure*}

\section{Conclusion}
We propose Overlapping Local Decentralized SGD (OLDSGD), a novel decentralized training algorithm that overlaps communication and computation to minimize synchronization overhead. Compared to Local DSGD, OLDSGD reduces per-iteration runtime to at most 50\% while maintaining the same theoretical guarantees under non-convexity. Experimentally, OLDSGD achieves a geometric mean speedup of 1.64\(\times\) over LDSGD and up to 3.32\(\times\) for GPT2 finetuning, excelling in communication-constrained settings (\(c = 5\)) prevalent in federated learning and edge computing. With minimal algorithmic modification, OLDSGD delivers significant runtime improvements without sacrificing theoretical robustness, making it an attractive drop-in solution for decentralized training.


\bibliographystyle{siam}
\bibliography{ref} 


\appendix

\section{Iteration Runtime Analysis}
Table \ref{tab:runtime} illustrates the per-iteration runtime of different algorithms, a detailed analysis on each algorithm's iteration-wise runtime is provided below. The setting considers $n$ agent, where each agent takes 1 unit of time to compute (stochastic) gradient and $c$ units of time to communicate data of the model size. Each communication round includes $\tau$ local steps. The original notations in the references are adopted when discussing.
\begin{itemize}
    \item LSGD: ring Allreduce divides model into $n$ chunks, transmitting each chunk $2(n-1)$ times, resulting in a $\frac{2(n-1)}{n}c$ communication time. The overall runtime $\tau + \frac{2(n-1)}{n}c$ follows after counting local gradient computations.
    \item OLDSGD: as illustrated in Fig. \ref{fig:schematic_OLDSGD}. 
    \item LDSGD: as illustrated in Fig. \ref{fig:schematic_OLDSGD}. 
    \item KGT: both the model $\mbf{x}_i^{(t)}$ and the accumulated gradient $\mbf{z}_i^{(t)}$. Though not mentioned in the reference, the communication of the model $\mbf{x}_i^{(t)}$ can actually be overlapped by gradient computations, as each update uses model from last communication round. This results in the $\max \{\tau,c\}$ term. However, the pass of the accumulated gradient $\mbf{z}_i^{(t)}$ must wait until an agent finishes local steps, adding a $c$ to the overall runtime.
    \item LUGT: this method communicates the model and the tracking variable, resulting in a $2c$ communication time. And no overlapping can be applied, resulting in the $\tau + 2c$ runtime.
    \item LED: similar to Exact Diffusion\cite{yuan2018exact}, LED's communication budget is identical to a model size, and no overlapping can be applied, leading to a $\tau + c$ runtime.
\end{itemize}

\section{Overlap Feasibility in LED/LUGT}
\label{appen:OLED_OLUGT}
As discussed in Section \ref{sec:alg_design}, the overlapping principle can be extended to LED and LUGT. Algorithms \ref{alg:OLGT} and \ref{alg:OLED} demonstrate (using their original notation, which may differ from our conventions) fully overlapped communication-computation implementations that preserve the same average update properties as their non-overlapping counterparts. However, this overlapping approach leads to significant performance degradation, as evidenced in Fig. \ref{fig:OLED_OLGT}, where the loss curves of ResNet18 training under a low data heterogeneity level is presented. The task is simpler due to less data heterogeneity, as evidenced by faster convergence of OLDSGD. While OLGT with $\tau = 1$ maintains comparable performance to standard LUGT, increasing to $\tau = 5$ results in substantial slowdowns that outweigh the per-iteration runtime benefits. Fig. \ref{fig:OLED_OLGT} also highlights LED's inherent divergence problems, which explains our exclusion of LED from the comparative analysis in Section \ref{sec:alg_comp}. OLED diverges even faster than vanilla LED in all tested configurations. These results demonstrate that naive overlapping implementations are infeasible for both LED and LUGT. The development of stable overlapping variants remains an open research challenge requiring more sophisticated approaches than direct method combination.

\begin{algorithm}
\caption{Overlapping LUGT (OLGT)}
\begin{algorithmic}[1]
\label{alg:OLGT}
\STATE \textbf{Input:} $x_i^0 = 0 \in \mathbb{R}^m$, $y_i^0 = \alpha \nabla f_i(x_i^0)$, $\alpha > 0$, $\eta > 0$, $T_0 \in \mathbb{Z}_{\geq 0}$, $K \in \mathbb{Z}_+$
\STATE \textbf{Define:} $\tau = \{0, T_0, 2T_0, 3T_0, \ldots\}$
\FOR{$k = 0$ \textbf{to} $K-1$}
    \IF{$k \in \tau$}
        \STATE $x_i^{k+1} = \sum_{j \in \mathcal{N}_i} w_{ij}  x_j^{k+1-T_0}  - \eta \sum_{t=k+1-T_0}^k y_i^t $
        \STATE $y_i^{k+1} = \sum_{j \in \mathcal{N}_i} w_{ij}  y_j^{k+1-T_0}  + \alpha \sum_{t=k+1-T_0}^k (\nabla f_i(x_i^{t+1}) - \nabla f_i(x_i^{t}))$
    \ELSE
        \STATE $x_i^{k+1} = x_i^k - \eta y_i^k$
        \STATE $y_i^{k+1} = y_i^k + \alpha \nabla f_i(x_i^{k+1}) - \alpha \nabla f_i(x_i^k)$
    \ENDIF
\ENDFOR
\end{algorithmic}
\end{algorithm}

\begin{algorithm}
\caption{Overlapping Local Exact-Diffusion (OLED)}
\begin{algorithmic}[1]
\label{alg:OLED}
\STATE \textbf{Input:} $x_i^0$, $\alpha > 0$, $\beta > 0$, and $\tau$.
\STATE \textbf{Initialize:} $y_i^0 = x_i^0 - \sum_{j\in\mc{N}_i }w_{ij}x_j^0$ (or $y_i^0 = 0$)
\FOR{$r = 0, 1, 2, \ldots$} 
    \STATE Set $\phi^r_{i,0} = x_i^r$
    \FOR{$t = 0$ \textbf{to} $\tau-1$} 
        \STATE $\phi^r_{i,t+1} = \xi^r_{i,t} -\alpha \nabla F_i(\phi_{i,t}^r;\xi_{i,t}^r) -  \beta y_i^r$ 
    \ENDFOR
    \STATE $x_i^{r+1} = \sum_{j \in \mathcal{N}_i} w_{ij} x_j^{r} + (\phi_{i,\tau}^r - x_i^r)$ 
    \STATE $y_i^{r+1} = y_i^r + \nabla f_i(x_i^{r+1}) - \nabla f_i(x_i^r)$ 
\ENDFOR
\end{algorithmic}
\end{algorithm}

\begin{figure}
    \centering
    \includegraphics[width=0.7\linewidth]{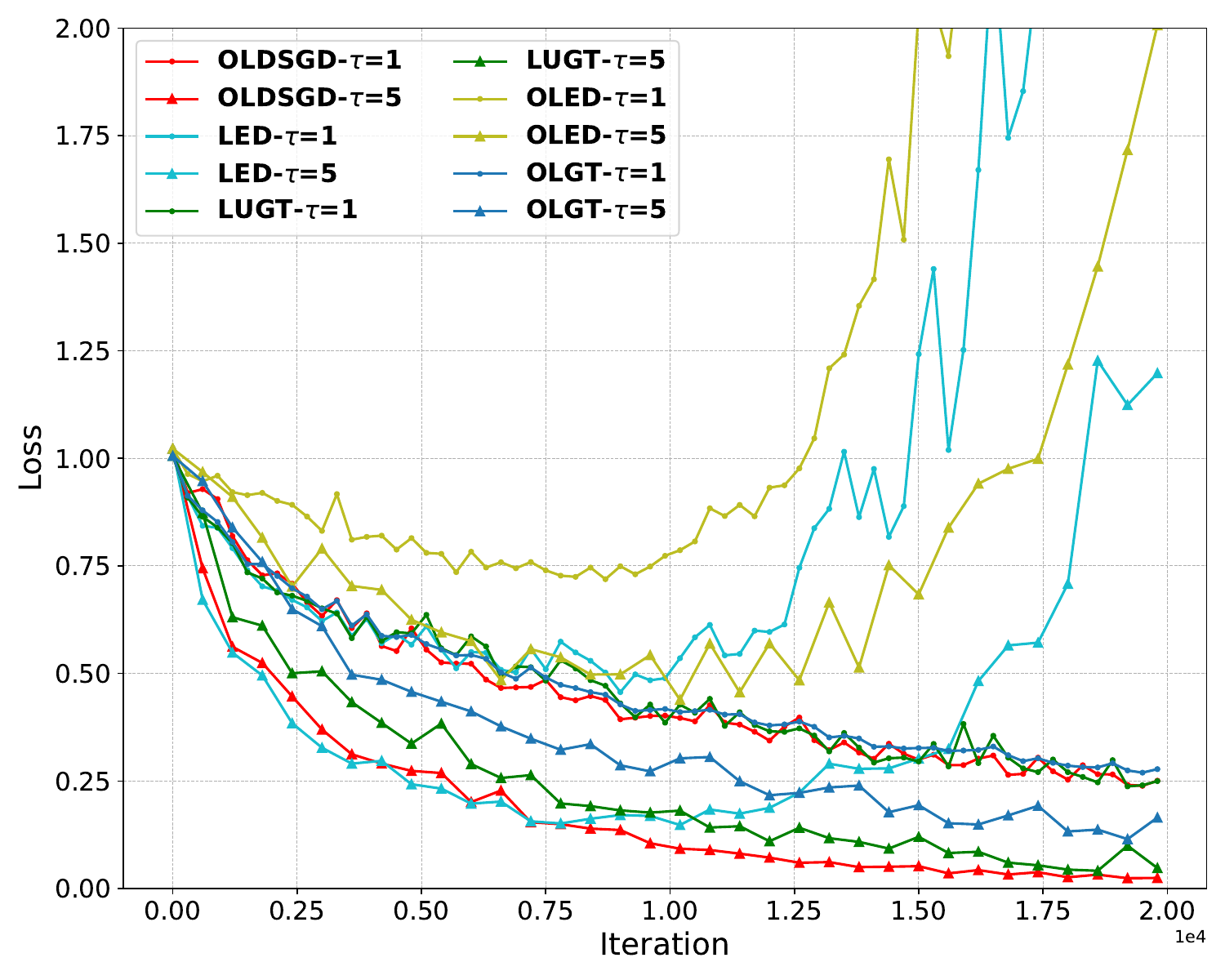}
    \caption{Loss curves of ResNet18 training w.r.t. iteration under a low data heterogeneity level. LED/OLED diverges in all cases and OLGT performs worse than LUGT under local steps.}
    \label{fig:OLED_OLGT}
\end{figure}

\section{Theoretical Convergence of OLDSGD}
\label{sec:proof}
In this section, we theoretically proof the convergence of Overlapping Local DSGD under non-convexity. 

\subsection{Notation}
We present notations used in the following analysis. Readers may also refer to Table \ref{tab:notations} as a complement.

First, we define that for any iteration $t$,
\begin{equation*}
    t' \triangleq (\lfloor\frac{t}{\tau}\rfloor - 1)\tau \quad \text{and}\quad t'' \triangleq \lfloor\frac{t}{\tau}\rfloor\tau,
\end{equation*}
where $t''$ is the most recent consensus iteration. At $t''$, the models models updated at $t'$ are used. Due to overlapping, the recursive range of OLDSGD expands from $\tau$ to $2\tau$, as we will demonstrate below.

For a compact matrix representation, we define
\begin{equation*}
    X_t = [x_t^1,...,x_t^n]\quad \text{and} \quad G_t = [g_t^1,..., g_t^n],
\end{equation*}
where $g_t^i = \nabla F(x_t^i, \xi_t^i)$.

We define the expected consensus error to be 
\begin{equation*}
    E_t \triangleq \E \|X_t - \bar x_t\1_n^T\|^2_F,
\end{equation*}
where $\|\cdot\|_F$ is the Frobenius norm.

Let $\lambda_2$ denote the eigenvalue of $W$ with the second largest absolute value. The one-step contraction rate of the consensus error after mixing is given by $p = (1-\lambda_2)^2$.


\subsection{Convergence analysis}
We first introduce several useful equalities from update \eqref{eq:OLDSGD}.
First, we reiterate the update of the average model
\begin{equation}
\label{eq:appen_ave_up}
    \bar x^{t+1} = \bar x^t - \frac{\alpha}{n}\sum_{i=1}^n g^t_i.
\end{equation}
As elaborated in the main text, the average update resembles SGD, despite overlapping communication and computation. Therefore, staleness incurred by overlapping does not affect the average update, avoiding the gradient mismatch problem in most distributed methods \cite{zhu2021delayed,kale2025eager,chen2023workie}.

We then present the relation between models at arbitrary iteration and $t'$, 
\begin{equation}
\label{eq:t_t'}
    x_t^i = \sum_{j}w_{ij} x_{t'}^j - \alpha\sum_{k=t'}^{t-1} g_k^i
\end{equation}
and
\begin{equation}
\label{eq:t_t'_ave}
    \bar x_t = \bar x_{t'} - \frac{\alpha}{n}\sum_{k=t'}^{t-1}\sum_{i=1}^n g_k^i.
\end{equation}

Equation \eqref{eq:t_t'} and \eqref{eq:t_t'_ave} depicts how models evolve within a period of $2\tau$, and is critical for analyzing the consensus error. In some sense, OLDSGD is pretty similar to Local DSGD with $2\tau$ local steps. However, OLDSGD is not covered in the framework proposed by \cite{koloskova2020unified}, which is 
\begin{equation*}
    X^{t+1} = (X^t-G^t)W^t,
\end{equation*}
with $W^t$ being $I$ or $W$, periodically. This is because such a framework adopts fresh neighbor models, compared to staled models $\{x_{t'}^j\}_{j\in\mc{N}_i}$ used by OLDSGD.

From \eqref{eq:appen_ave_up}, it is straight forward to derive the following descent lemma.

\begin{lemma}(Descent lemma)
\label{le:descent}
    Given Assumption \ref{as:L_smooth}, \ref{as:var}, and \ref{as:data_hete}, we have
    $$
    \E_{t+1} f(\bar x^{t+1}) \le f(\bar x^{t}) - \frac{\alpha}{4} \|\nabla f(\bar x^t)\|^2 + \frac{2\alpha L^2}{n} \sum_{i=1}^n \|x_i^t - \bar x^t\| + \frac{\alpha^2L}{n}(\frac{\sigma^2}{2} +M\zeta^2),$$
    where $\alpha \le \min\{\frac{1}{L}, \frac{n}{4LM(P+1)},\frac{n}{LM} \}$.
\end{lemma}
\begin{proof}
From \eqref{eq:appen_ave_up}, we have
\begin{align*}
    &\E_{t+1} f(\bar x^{t+1}) \\
    &= \E_{t+1} f\left(\bar x^{(t)} - \frac{\alpha}{n}\sum_{i=1}^n\nabla F(x_i^t, \xi_i^t)\right)\\
    &\le f(\bar x^{t}) - \alpha \E_{t+1} \langle \nabla f(\bar x^t), \frac{1}{n}\sum_{i=1}^n\nabla F(x_i^t, \xi_i^t)\rangle + \frac{\alpha^2 L}{2} \E_{t+1} \left\|\frac{1}{n}\sum_{i=1}^n\nabla F(x_i^t, \xi_i^t)\right\|^2\\
    &\le f(\bar x^{t}) - \alpha \langle \nabla f(\bar x^t), \frac{1}{n}\sum_{i=1}^n\nabla f(x_i^t)\rangle + \frac{\alpha^2 L}{2}\left( \E \|\frac{1}{n}\sum_{i=1}^n(\nabla F(x_i^t, \xi_i^t)- \nabla f_i(x_i^t))\|^2 +  \|\frac{1}{n}\sum_{i=1}^n\nabla f_i(x_i^t)\|^2\right)\\
    &\le f(\bar x^{t}) - \frac{\alpha}{2} \|\nabla f(\bar x^t)\|^2 - \frac{\alpha}{2}(1-\alpha L) \|\frac{1}{n}\sum_{i=1}^n\nabla f_i(x_i^t)\|^2 + \frac{\alpha}{2} \|\nabla f(\bar x^t) - \frac{1}{n}\sum_{i=1}^n\nabla f_i(x_i^t)\|^2\\
    &+\frac{\alpha^2 L}{2} \E \|\frac{1}{n}\sum_{i=1}^n(\nabla F(x_i^t, \xi_i^t)- \nabla f_i(x_i^t))\|^2 \\
    &\overset{\alpha \le 1/L}{\le} f(\bar x^{t}) - \frac{\alpha}{2} \|\nabla f(\bar x^t)\|^2 + \frac{\alpha L^2}{n} \sum_i \|x_i^t - \bar x^t\| +\frac{\alpha^2 L}{2n^2} \E_{t+1}\sum_{i=1}^n \|(\nabla F(x_i^t, \xi_i^t)- \nabla f_i(x_i^t))\|^2\\
    &\le f(\bar x^{t}) - \frac{\alpha}{2} \|\nabla f(\bar x^t)\|^2 + \frac{\alpha L^2}{n} \sum_i \|x_i^t - \bar x^t\| + \frac{\alpha^2 LM}{2n^2} \sum_{i} \|\nabla f_i(x_i^k)\|^2 + \frac{\alpha^2 L\sigma^2}{2n},
\end{align*}
where the third inequality is from the polarization identity, and the fourth and fifth inequalities are from Assumption \ref{as:L_smooth} and \ref{as:var}.

From Assmption \ref{as:data_hete}, we have
\begin{align*}
    \sum_{i} \|\nabla f_i(x_i^t)\|^2 &= \sum_{i} \|\nabla f_i(x_i^t) \mp \nabla f_i(\bar x^t)\|^2 \\
    &\le 2L^2\sum_i \|x_i^t - \bar x^t\|^2 + 2\sum_i \|\nabla f_i(\bar x^t)\|^2\\
    &\le 2L^2\sum_i \|x_i^t - \bar x^t\|^2 + 2n \zeta^2 + 2n(P+1)\|\nabla f(\bar x^t)\|^2,
\end{align*}
where $P+1$ arises from $\|\nabla f_i(\bar x^t)\| \le 2\|\nabla f_i(\bar x^t)-\nabla f(\bar x^t)\|+ 2\|\nabla f(\bar x^t)\|$.

Combine everything
\begin{align*}
     &\E_{t+1} f(\bar x^{t+1}) \\
    &\le f(\bar x^{t}) - \alpha(\frac{1}{2}-\frac{\alpha LM(P+1)}{n}) \|\nabla f(\bar x^t)\|^2 + (\frac{\alpha L^2}{n} + \frac{\alpha^2 L^3M}{n^2}) \sum_i \|x_i^t - \bar x^t\| + \frac{\alpha^2}{n}(\frac{ L\sigma^2}{2} + LM\zeta^2)\\
    &\overset{\alpha \le \frac{n}{4LM(P+1)}}{\le} f(\bar x^{t}) - \frac{\alpha}{4} \|\nabla f(\bar x^t)\|^2 + (\frac{\alpha L^2}{n} + \frac{\alpha^2 L^3M}{n^2}) \sum_i \|x_i^t - \bar x^t\| + \frac{\alpha^2L}{n}(\frac{\sigma^2}{2} +M\zeta^2)\\
    &\overset{\alpha \le \frac{n}{LM}}{\le} f(\bar x^{t}) - \frac{\alpha}{4} \|\nabla f(\bar x^t)\|^2 + \frac{2\alpha L^2}{n} \sum_i \|x_i^t - \bar x^t\| + \frac{\alpha^2L}{n}(\frac{\sigma^2}{2} +M\zeta^2)
\end{align*}
\end{proof}

To bound the RHS of Lemma \ref{le:descent}, we need to bound the consensus error. Recall that in OLDSGD, at each iteration $t$, each agent conducts a consensus step at $t''$, using stale models updated at $t'$. Therefore, we start with a recursion between models at $t'$ and $t$.

\begin{lemma}\label{le:con_1}(consensus error part 1)
Given Assumption \ref{as:d_sto}-\ref{as:data_hete}, when $\alpha \le \frac{p}{16L\sqrt{3\tau(2\tau+M)}}$ we have
\begin{align*}
    E_t \le  (1-\frac{p}{2}) E_{t'} + \frac{p}{64\tau} \sum_{k=t'}^{t-1} E_k+ C \E \sum_{k=t'}^{t-1}\alpha^2\|\nabla f(\bar x_k) \|^2 + D\sum_{k=t'}^{t-1}\alpha^2,
\end{align*}
where $C=\frac{12(2\tau+M)n(P+1)}{p}, D=\frac{6((2\tau+M)n\zeta^2+ n\sigma^2)}{p}$, and $p = 1-\lambda_2^2$.
\end{lemma}

\begin{proof}
 By the equation \eqref{eq:t_t'} and \eqref{eq:t_t'_ave}, we have
 $$
 X_t - \bar x_t\1_n^T = (X_{t'}W - \bar x_{t'}\1_n^T) - \alpha \sum_{k=t'}^{t-1} G_k(I-J),
 $$
 where $J=\frac{\1_n \1_n^T}{n}$ and $t'=(\lfloor\frac{t}{\tau}\rfloor-1)\tau$.
Taking square norm and expectation
\begin{align}
\label{eq:con_recursion_1}
    \E \|X_t - \bar x_t\1_n^T\|^2_F &= \E \|(X_{t'}W - \bar x_{t'}\1_n^T) - \alpha \sum_{k=t'}^{t-1} G_k(I-J)\|^2_F\nonumber\\
    &\le \lambda_2^2(1+\epsilon) \E \|X_{t'} - \bar x_{t'}\1_n^T\|^2_F + (1+\frac{1}{\epsilon})\alpha^2 \underbrace{\E \|\sum_{k=t'}^{t-1} G_k(I-J)\|^2_F}_{T_1}
\end{align}

We bound $T_1$ as follows
\begin{align*}
    T_1 &= \E \sum_{i=1}^n \left\|\sum_{k=t'}^{t-1} \nabla F(x_k^i,\xi_k^i)-\frac{1}{n}\sum_{k=t'}^{t-1}\sum_{j=1}^n\nabla F(x_k^j,\xi_k^j) \right\|^2\\
    &= \E \sum_{i=1}^n \left\|\sum_{k=t'}^{t-1} \nabla F(x_k^i,\xi_k^i)\right\|^2-n\E\left\|\frac{1}{n}\sum_{k=t'}^{t-1}\sum_{j=1}^n\nabla F(x_k^j,\xi_k^j) \right\|^2\\
    &\le \E \sum_{i=1}^n\left\|\sum_{k=t'}^{t-1} \nabla F(x_k^i,\xi_k^i)\right\|^2\\
    &\le 2\E \sum_{i=1}^n\left\|\sum_{k=t'}^{t-1}( \nabla F(x_k^i,\xi_k^i)- \nabla f_i(x_k^i))\right\|^2 + 2\E\sum_{i=1}^n\left\|\sum_{k=t'}^{t-1}\nabla f_i(x_k^i)\right\|^2 \\
    &= 2\E \sum_{i=1}^n\sum_{k=t'}^{t-1}\left\|( \nabla F(x_k^i,\xi_k^i)- \nabla f_i(x_k^i))\right\|^2 + 2\E\sum_{i=1}^n\left\|\sum_{k=t'}^{t-1}\nabla f_i(x_k^i)\right\|^2 \\
    &\le 2\sum_{k=t'}^{t-1}(n\sigma^2 + M \sum_i \|\nabla f_i(x_k^i)\|^2) + 2\E\sum_{i=1}^n\left\|\sum_{k=t'}^{t-1}\nabla f_i(x_k^i)\right\|^2 \\
    &\overset{t'-t\le 2\tau}{\le} (4\tau + 2M) \sum_{k=t'}^{t-1}\underbrace{\sum_i \|\nabla f_i(x_k^i)\|^2}_{T_2} + 2\sum_{k=t'}^{t-1} n\sigma^2,
\end{align*}
where the third equality is because $\zeta_t^i$ is independent with $x_{t-1}^i$.
We further bound $T_2$ as follows
\begin{align*}
    T_2 &\le 2\E\sum_i\left(\left\|\left(\nabla f_i(x_k^i)- \nabla f_i(\bar x_k)\right)\right\|^2+ \left\|\nabla f_i(\bar x_k)\right\|^2 \right)\\
    &\le 2L^2\E\|X_k - \bar x_k \mbf{1}_n^T\|^2_F + 2n(\zeta^2+(P+1)\|\nabla f(\bar x_k)\|^2)\\
    &= 2L^2 E_k + 2n(P+1) \|\nabla f(\bar x_k)\|^2+2n\zeta^2
\end{align*}
Therefore, we have 
\begin{align*}
    T_1 \overset{t-t'\le 2\tau}{\le} 4(2\tau + M)L^2\sum_{k=t'}^{t-1} E_k + 4n(2\tau+M)(P+1)\sum_{k=t'}^{t-1}\|\nabla f(\bar x_k)\|^2 + 2n(2\tau+M)\sum_{k=t'}^{t-1}\zeta^2 + 2\sum_{k=t'}^{t-1} n\sigma^2
\end{align*}

Together with equation (\ref{eq:con_recursion_1}) we have
\begin{align*}
     E_t
    &\le \lambda_2^2(1+\epsilon) E_{t'} \\
    &+ 2(1+\frac{1}{\epsilon})\alpha^2 \left(2(2\tau+M) L^2\sum_{k=t'}^{t-1} E_k+ 2(2\tau+M) n(P+1) \E \sum_{k=t'}^{t-1}\|\nabla f(\bar x_k) \|^2+ n(2\tau+M)\sum_{k=t'}^{t-1}\zeta^2 + \sum_{k=t'}^{t-1} n\sigma^2\right)\\
    & \overset{\epsilon = \frac{1-\lambda_2^2}{2}}{\le} (1-\frac{p}{2}) E_{t'} \\
    &+ \frac{6}{p}\alpha^2  \left(2(2\tau+M) L^2\sum_{k=t'}^{t-1} E_k+ 2(2\tau+M) n(P+1) \E \sum_{k=t'}^{t-1}\|\nabla f(\bar x_k) \|^2 + n(2\tau+M)\sum_{k=t'}^{t-1}\zeta^2 + \sum_{k=t'}^{t-1} n\sigma^2 \right)\\
    &\overset{\alpha \le \frac{p}{16L\sqrt{3\tau(2\tau+M)}}}{\le} (1-\frac{p}{2}) E_{t'} + \frac{p}{64\tau} \sum_{k=t'}^{t-1} E_k+ C \E \sum_{k=t'}^{t-1}\alpha^2\|\nabla f(\bar x_k) \|^2 + D\sum_{k=t'}^{t-1}\alpha^2,
\end{align*}
where $p\triangleq1-\lambda_2^2$, the second inequality is from $(1-2x)(1+x) \le 1-x$ and $p< 1$, and $C=\frac{12(2\tau+M)n(P+1)}{p}, D=\frac{6((2\tau+M)n\zeta^2+ n\sigma^2)}{p}$.

     
\end{proof}
    
To bound the sum of consensus error, we need to unroll $E_t$ to previous time steps. However, lemma \ref{le:con_1} only characterizes the case when there is a consensus step between $t$ and $t'$. We need a similar lemma for to unroll $E_t$ to $E_{t''}$, where $t''=\lfloor \frac{t}{\tau}\rfloor\tau$. That is, there is no consensus step between $t$ and $t''$.

\begin{lemma}\label{le:con_2}(consensus error part 2)
Given Assumption \ref{as:L_smooth}, \ref{as:var}, and \ref{as:data_hete}, when $\alpha \le \frac{p}{16L\sqrt{3\tau(2\tau+M)}}$, we have
\begin{align*}
   E_t &\le (1+\frac{p}{2}) E_{t''} + \frac{p}{64\tau} \sum_{k=t''}^{t-1} E_k+ C \E \sum_{k=t''}^{t-1}\alpha^2\|\nabla f(\bar x_k) \|^2 + D\sum_{k=t''}^{t-1}\alpha^2
\end{align*}
where $t''=\lfloor \frac{t}{\tau}\rfloor\tau$.
    
\end{lemma}

\begin{proof}
 Similar to the proof of Lemma \ref{le:con_1}
 $$
 X_t - \bar x_t\1_n^T = (X_{t''} - \bar x_{t''}\1_n^T) - \alpha \sum_{k=t''}^{t-1} G_k(I-J),
 $$
 where $J=\frac{\1_n \1_n^T}{n}$ and $t''=\lfloor\frac{t}{\tau}\rfloor\tau$.
Taking square norm and expectation
\begin{align}
\label{eq:con_recursion_2}
    \E \|X_t - \bar x_t\1_n^T\|^2_F &= \E \|(X_{t''} - \bar x_{t''}\1_n^T) - \alpha \sum_{k=t''}^{t-1} G_k(I-J)\|^2_F\nonumber\\
    &\le (1+\epsilon) \E \|X_{t''} - \bar x_{t''}\1_n^T\|^2_F + (1+\frac{1}{\epsilon})\alpha^2 \underbrace{\E \|\sum_{k=t''}^{t-1} G_k(I-J)\|^2_F}_{T_3}
\end{align}
$T_3$ can be similarly bounded, with the only difference being $t-t''\le \tau$,
\begin{align*}
    T_3 \overset{t-t''\le \tau}{\le} 4(\tau + M)L^2\sum_{k=t''}^{t-1} E_k + 4n(\tau+M)(P+1)\sum_{k=t'}^{t-1}\|\nabla f(\bar x_k)\|^2 + 2n(\tau+M)\tau\zeta^2 + 2\tau n\sigma^2
\end{align*}
Thus, when $\alpha \le \frac{p}{16\sqrt{3\tau(2\tau+M)}}$ and $\epsilon = \frac{1-\lambda_2^2}{2}$, 
\begin{align*}
    E_t &\le (1+\frac{p}{2}) E_{t'} + \frac{p}{\tau} \sum_{k=t''}^{t-1} E_k+ C \E \sum_{k=t''}^{t-1}\alpha^2\|\nabla f(\bar x_k) \|^2 + D\sum_{k=t''}^{t-1}\alpha^2.
\end{align*}

\end{proof}

Combining Lemma \ref{le:con_1} and \ref{le:con_2}, we arrive at the following bound on the overall consensus error.

\begin{lemma}(overall consensus error)
\label{le:overall_con}
Given Assumption \ref{as:L_smooth}, \ref{as:var}, and \ref{as:data_hete}, we have
    \begin{align*}
        \frac{2L^2}{n}\sum_{t=0}^{T-1}E_t \le \frac{1}{8} \sum_{t=0}^{T-1}\E \|\nabla f(\bar x^t)\|^2 + \frac{128L^2}{n}\frac{DT\tau}{p}\alpha^2,
    \end{align*}
    where $\alpha \le \min\{\frac{p}{16L\sqrt{3\tau(2\tau+M)}}, \frac{1}{32L} \sqrt{\frac{pn} {2C\tau}}\}$.
\end{lemma}

\begin{proof}
    Applying Lemma 14 in \cite{koloskova2020unified} to Lemma \ref{le:con_1} and \ref{le:con_2}, 
    we have
    \begin{align*}
        B\sum_{t=0}^{T-1}w_t E_t \le \frac{1}{8} \sum_{t=0}^{T-1} w_t\E\|\nabla f(\bar x^t)\|^2 + 64BD\frac{\tau}{p}\sum_{t=0}^{T-1} w_t \alpha^2,
    \end{align*}
    where $\alpha \le \frac{1}{32} \sqrt{\frac{p} {BC\tau}}$ and $\{w_t\}_t$ is $\frac{16\tau}{p}$-slow-increasing.

    Set $B=\frac{2L^2}{n}, w_t=1$, we concluded.
    
\end{proof}

The key observation in Lemma \ref{le:overall_con} is that we can bound the cumulative consensus error by the cumulative gradient square norm. Now we apply Lemma \ref{le:overall_con} to Lemma \ref{le:descent} for the convergence theorem.

\begin{theorem}[Theorem \ref{th:overall_con}]
Given Assumption \ref{as:d_sto}-\ref{as:lb}, we have
    \begin{align*}
        \frac{\sum_{t=0}^{T-1} \E \|\nabla f(\bar x^t)\|^2}{T} &\le \frac{8(f^0 - f^*)}{\alpha T} + \frac{8\alpha L}{n}(\frac{\sigma^2}{2} + M\zeta^2)+ \frac{1024 L^2}{n}\frac{D\tau}{p}\alpha^2,
    \end{align*}
    where $\alpha \le \min\{\frac{1}{L}, \frac{n}{4LM(P+1)},\frac{n}{LM},\frac{p}{16L\sqrt{3\tau(2\tau+M)}}, \frac{1}{32L} \sqrt{\frac{pn} {2C\tau}}\}$, and $C, D$ are defined in Lemma \ref{le:con_1}.
\end{theorem}

\begin{proof}
    Take full expectation of Lemma \ref{le:descent} and sum over $t$, we have
    \begin{align*}
        \frac{\alpha}{4} \frac{\sum_{t=0}^{T-1} \E \|\nabla f(\bar x^t)\|^2}{T} \le \frac{f^0 - f^*}{T} + \frac{2\alpha L^2}{n} \frac{\sum_{t=0}^{T-1} E_t}{T} + \frac{\alpha^2L}{n}(\frac{\sigma^2}{2} + M\zeta^2)
    \end{align*}  
    By Lemma \ref{le:overall_con}, we have 
    \begin{align*}
        \frac{\alpha}{4} \frac{\sum_{t=0}^{T-1} \E \|\nabla f(\bar x^t)\|^2}{T} &\le \frac{f^0 - f^*}{T} + \frac{\alpha}{8} \frac{\sum_{t=0}^{T-1} \E \|\nabla f(\bar x^t)\|^2}{T}  + \frac{128L^2}{n}\frac{D\tau}{p}\alpha^3 + \frac{\alpha^2L}{n}(\frac{\sigma^2}{2} + M\zeta^2).
    \end{align*}
    Rearrange the terms, we concluded.
\end{proof}

\begin{corollary}[Corollary \ref{coro:1}]
Given identical assumptions in Theorem \ref{th:overall_con}, let $\alpha = \sqrt\frac{n}{T},$ we have 
\begin{align*}
    \frac{\sum_{t=0}^{T-1} \E \|\nabla f(\bar x^t)\|^2}{T} &\le \frac{8(f^0 - f^*)}{\sqrt{nT}} + \frac{8L(\frac{\sigma^2}{2} + M\zeta^2)}{\sqrt{nT}}+ \frac{6144n L^2\tau((2\tau+M))\zeta^2 + \sigma^2)}{p^2T},
\end{align*}
given
    \begin{align*}
        T \ge \max\{nL^2, \frac{16L^2M^2(P+1)^2}{n}, \frac{L^2M^2}{n}, \frac{768nL^2\tau(2\tau+M)}{p^2}, \frac{24576nL^2\tau(2\tau+M)(P+1)}{p^2}\}.
    \end{align*}
\end{corollary}

\begin{proof}
Substitute $\alpha = \sqrt\frac{n}{T}$ into Theorem \ref{th:overall_con}, we have
    \begin{align*}
        \frac{\sum_{t=0}^{T-1} \E \|\nabla f(\bar x^t)\|^2}{T} &\le \frac{8(f^0 - f^*)}{\sqrt{nT}} + \frac{8L(\frac{\sigma^2}{2} + M\zeta^2)}{\sqrt{nT}}+ \frac{6144n L^2\tau((2\tau+M))\zeta^2 + \sigma^2)}{p^2T} \\
        &= \mc{O}(\frac{1}{\sqrt{nT}} + \frac{n}{T}).
    \end{align*}

    To satisfy the step size rule, we need
    \begin{align*}
        \sqrt\frac{n}{T} \le \min\{\frac{1}{L}, \frac{n}{4LM(P+1)},\frac{n}{LM},\frac{p}{16L\sqrt{3\tau(2\tau+M)}}, \frac{1}{32L} \sqrt{\frac{pn} {2C\tau}}\},
    \end{align*}
    which results in 
    \begin{align*}
        T \ge \max\{nL^2, \frac{16L^2M^2(P+1)^2}{n}, \frac{L^2M^2}{n}, \frac{768nL^2\tau(2\tau+M)}{p^2}, \frac{24576nL^2\tau(2\tau+M)(P+1)}{p^2}\}.
    \end{align*}
\end{proof}

\section{More on Experiments}
\label{sec:more_exp}
Due to space constraints, the main text presents only limited experimental results. This section provides a comprehensive analysis to complement those findings.

\subsection{Compute Resources}
\label{sec:com_resource}
The experiments are conducted on a server with 8 NVIDIA 3090 GPU, each with 24GB memory. The overall experiment takes roughly over 1500 GPU hours. The heavy computation is primarily from the wide range of local steps $\tau$, as we investigated 8 possible values, and we conducted both homogeneous and heterogeneous cases for VGG11 and ResNet18 training.

\subsection{Iteration-wise Convergence Comparison}
To extend the simulated wall-clock time analysis in Section~\ref{sec:alg_comp}, we investigate the iteration-wise convergence of Overlapping Local Decentralized SGD (OLDSGD) against baselines, including Local DSGD (LDSGD), KGT, LUGT, Local SGD (LSGD), and Local Exact Diffusion (LED). Figures~\ref{fig:ite_conv_VGG}, \ref{fig:ite_conv_resnet}, and \ref{fig:ite_conv_gpt} illustrate the training loss, test accuracy, and test perplexity for VGG11, ResNet18, and GPT2 finetuning on CIFAR-10 and language datasets, respectively, under homogeneous and heterogeneous data distributions. As noted, GPT2 finetuning was conducted only under homogeneous settings due to computational constraints. For clarity, we focus on results for local steps \(\tau \in \{1, 3, 5, 10\}\), though experiments with larger \(\tau\) reveal consistent trends. The analysis confirms OLDSGD’s theoretical equivalence to LDSGD, highlights the stability of gradient descent-based methods, demonstrates significant speedups with increasing local steps, and reveals OLDSGD’s superior test performance, particularly for GPT2, underscoring its potential for modern transformer-based models.

For VGG11 and ResNet18 training, OLDSGD exhibits iteration-wise convergence closely aligned with LDSGD, particularly for smaller local steps, corroborating Theorem \ref{th:overall_con}, which establishes identical convergence rates under non-convex settings due to OLDSGD’s SGD-like average update. At \(\tau = 1\) both algorithms achieve nearly identical training losses after several hundred iterations in homogeneous settings, with minimal differences. For \(\tau \in \{5,10\}\), losses remain comparable, and in heterogeneous settings, OLDSGD tracks LDSGD closely, though slightly trailing at due to staleness errors. As training progresses and gradients diminish, this staleness effect fades, allowing OLDSGD to match LDSGD’s final loss precision, as predicted by its theoretical convergence rate. Test accuracy reflects similar trends, with OLDSGD achieving comparable performance to LDSGD in homogeneous settings but occasionally surpassing it under heterogeneity, particularly at \(\tau = 1\) in ResNet18 training, where OLDSGD yields higher accuracy.

In GPT2 finetuning, OLDSGD consistently outperforms all baselines in test perplexity, showcasing practical advantages despite theoretical equivalence to LDSGD. At \(\tau = 3\), OLDSGD achieves lower perplexity than LDSGD and KGT, with the gap widening for larger \(\tau\). At \(\tau = 5\) and \(\tau = 10\), OLDSGD’s perplexity curves drop more sharply, as shown in Figure~\ref{fig:ite_conv_gpt}, maintaining a clear lead over LDSGD, KGT, and LSGD across all iterations. This consistent superiority, evident for all tested \(\tau\), highlights OLDSGD’s efficacy for transformer-based architectures, where communication bottlenecks are critical. 

Gradient descent-based methods, including OLDSGD, LDSGD, and LSGD, demonstrate remarkable stability across local steps compared to gradient tracking or exact diffusion approaches. For VGG11 and ResNet18, these methods maintain consistent loss reduction and accuracy gains at \(\tau = 10\), with no divergence observed even for larger \(\tau\). In contrast, KGT converges slowly under heterogeneity, with noticeably higher losses, while LUGT diverges entirely at \(\tau = 10\) in heterogeneous VGG11 training, as shown in Figure~\ref{fig:ite_conv_VGG}. LED similarly fails for \(\tau \geq 10\) in both tasks, underscoring the limitations of tracking-based methods \footnote{Tracking based methods also require more memory and communication budget.} despite their theoretical guarantees, such as exact convergence under heterogeneity.

OLDSGD’s speedup with increasing local steps is a key strength, observed across all tasks. In VGG11 and ResNet18 training, OLDSGD reduces losses faster per iteration as \(\tau\) increases, both in homogeneous and heterogeneous settings. For GPT2, the speedup is even more pronounced, with significantly faster perplexity reduction at \(\tau = 10\) compared to smaller \(\tau\). Experiments with larger \(\tau\), such as \(\tau = 20\), confirm this trend, suggesting OLDSGD can handle communication delays far exceeding computation time, making it suitable for real-world applications with substantial network latency.

These findings position OLDSGD as a robust and scalable solution for decentralized training, even in highly heterogeneous and communication-delayed environments. Its convergence closely tracks LDSGD’s, validating theoretical predictions, while its occasional superiority in test accuracy for VGG11 and ResNet18, and consistent dominance in GPT2 perplexity, highlight its practical advantages. The stability of GD-based methods, combined with OLDSGD’s ability to mask communication delays, makes it particularly suited for modern transformer-based models and federated learning scenarios.

\begin{figure*}[htbp]
\centering
\begin{subfigure}[t]{0.4\textwidth}
\centering
\includegraphics[width=\textwidth]{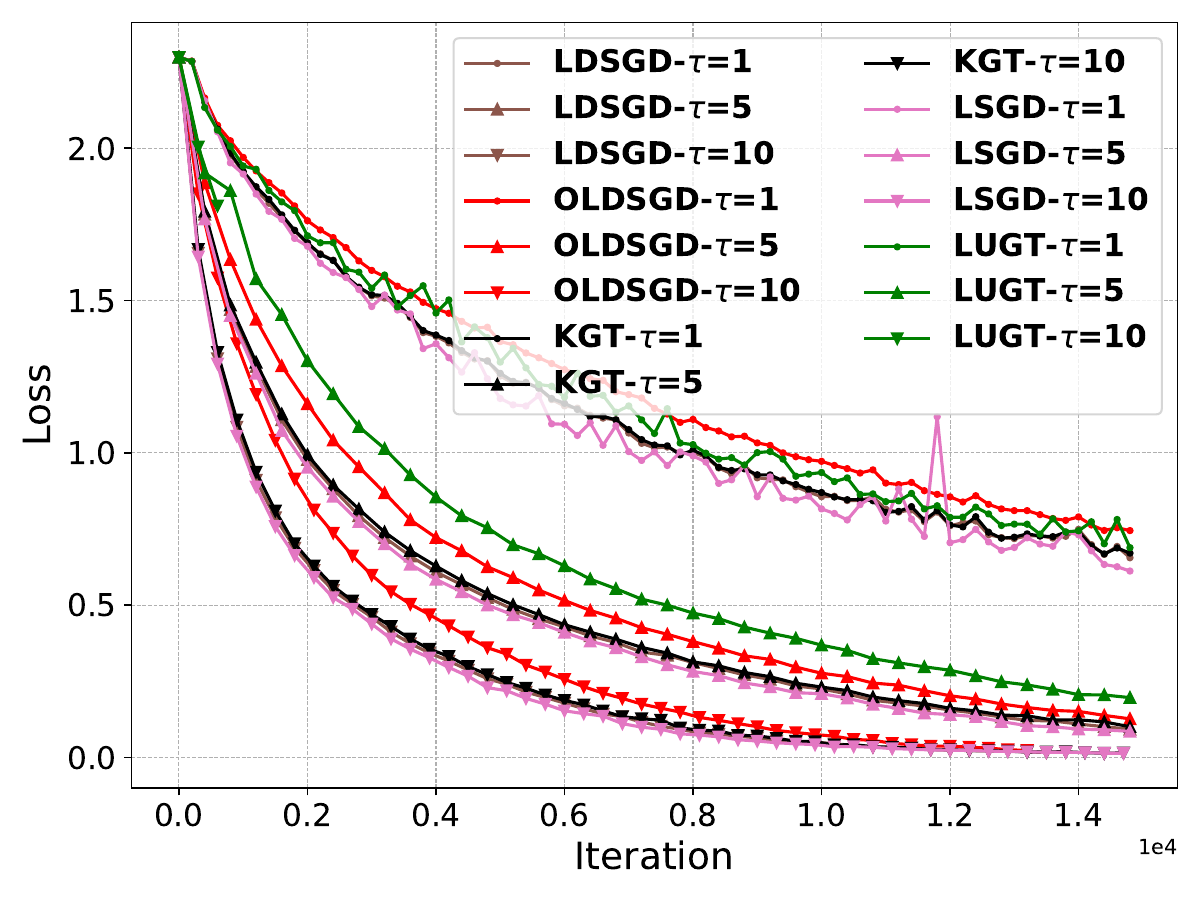}
\caption{Training Loss - Homo.} %
\end{subfigure}
\hspace{-0.2cm}
\begin{subfigure}[t]{0.4\textwidth}
\centering
\includegraphics[width=\textwidth]{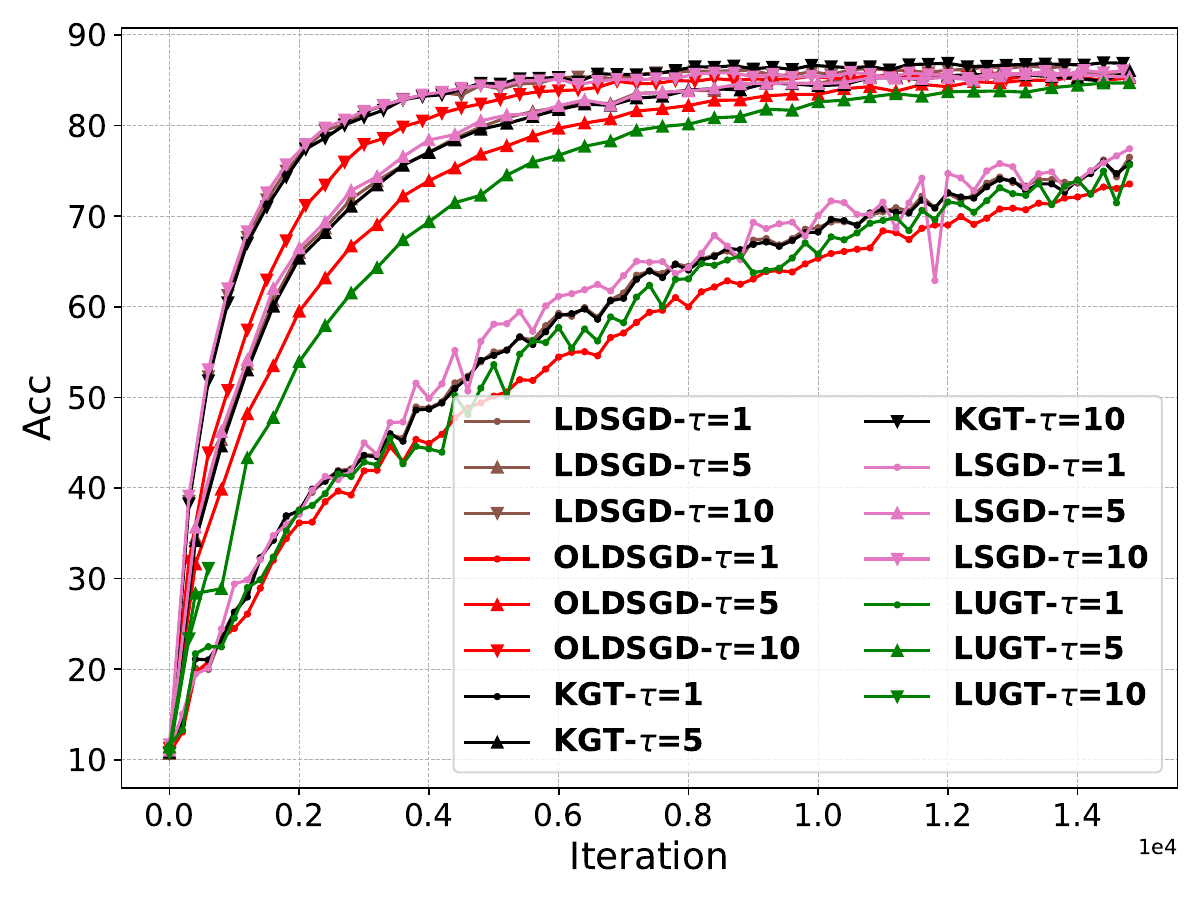}
\caption{Test Accuracy - Homo.}
\end{subfigure}
\begin{subfigure}[t]{0.4\textwidth}
\centering
\includegraphics[width=\textwidth]{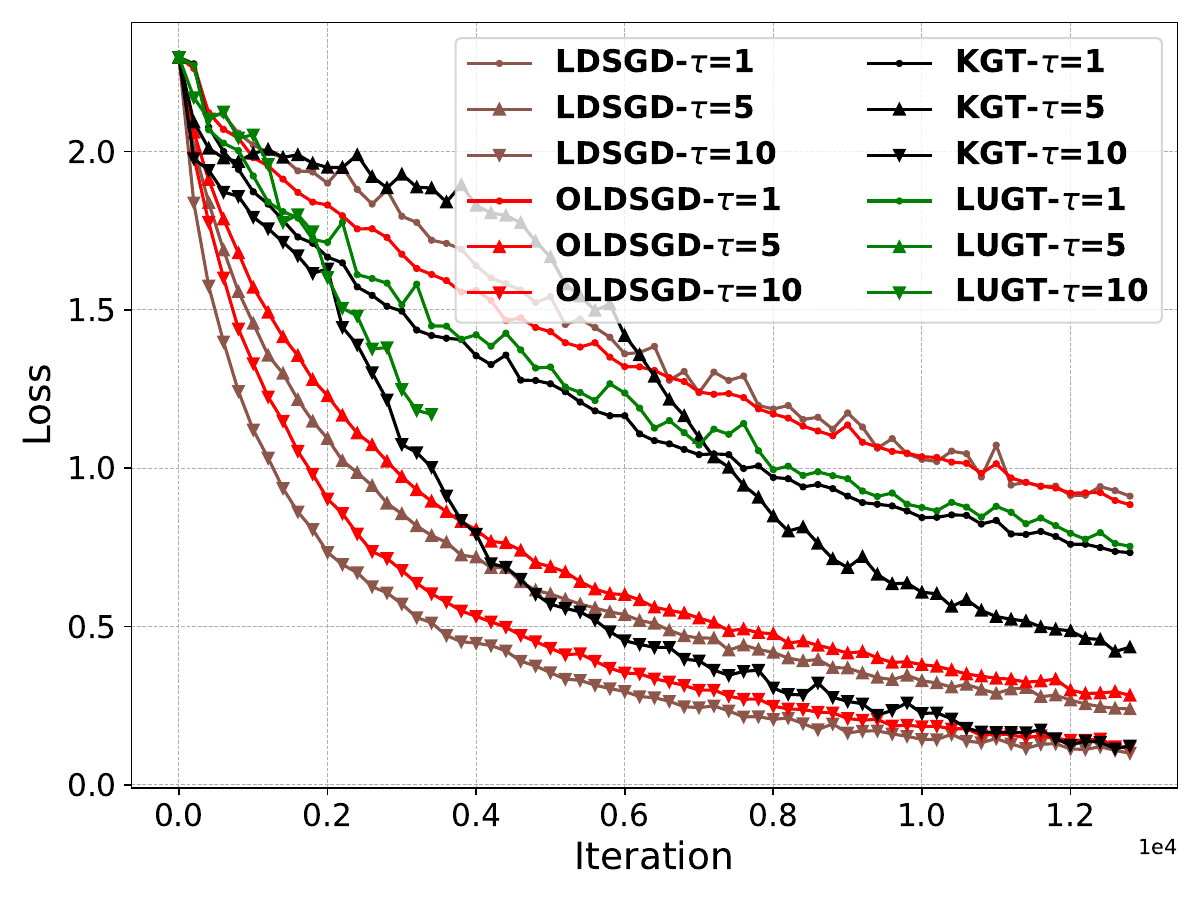}
\caption{Training Loss - Hete.} %
\end{subfigure}
\hspace{-0.2cm}
\begin{subfigure}[t]{0.4\textwidth}
\centering
\includegraphics[width=\textwidth]{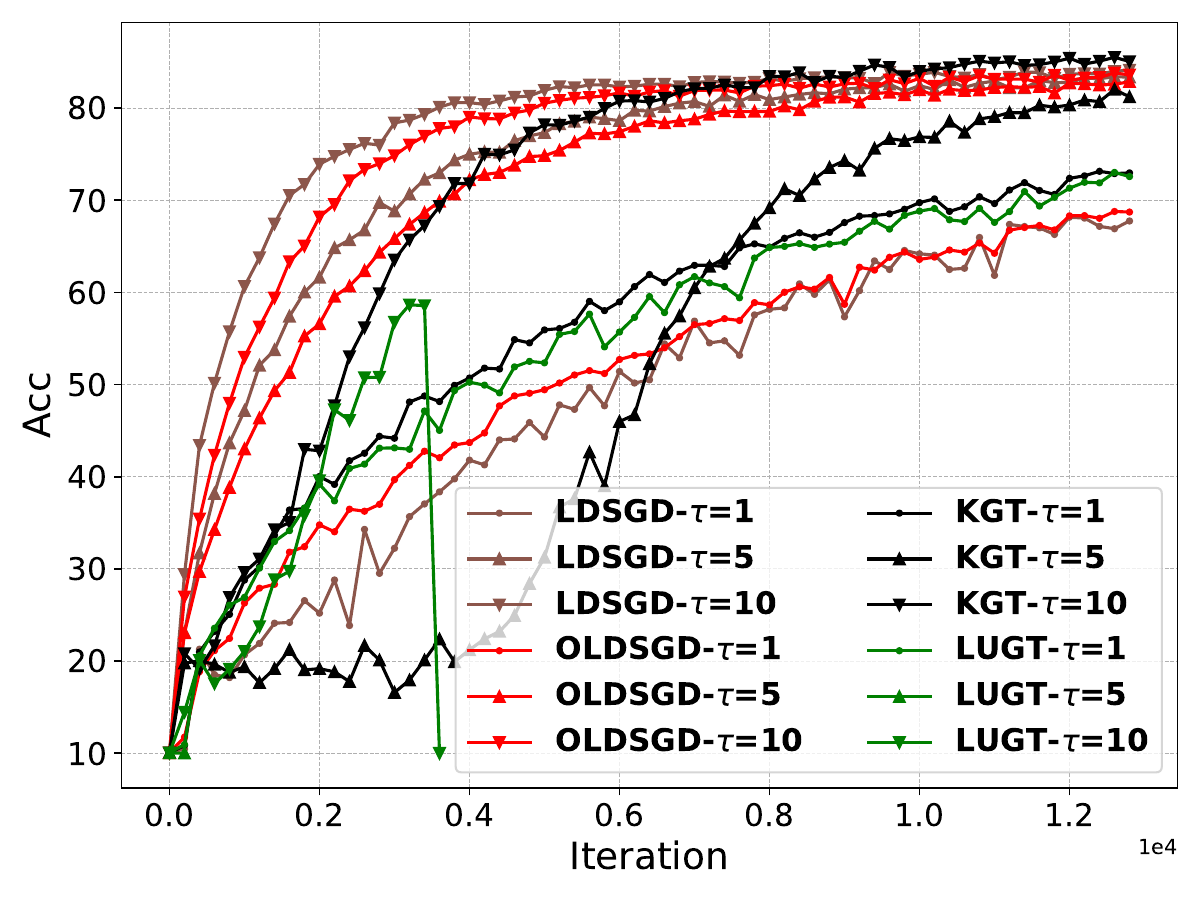}
\caption{Test Accuracy - Hete.}
\end{subfigure}
\caption{Iteration-wise convergence of VGG11 training. Different numbers of local steps are presented.}

\label{fig:ite_conv_VGG}
\end{figure*}

\begin{figure*}[htbp]
\centering
\begin{subfigure}[t]{0.4\textwidth}
\centering
\includegraphics[width=\textwidth]{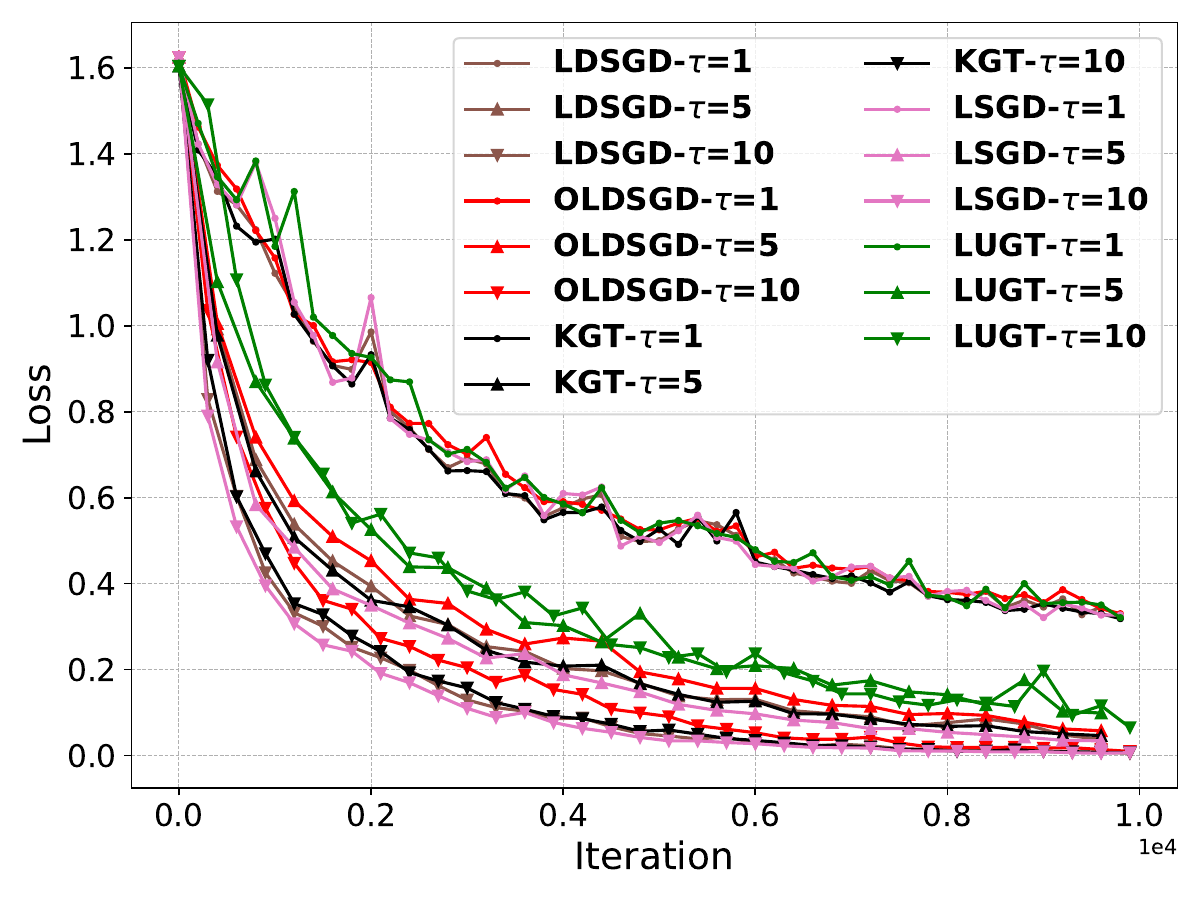}
\caption{Training Loss - Homo.} %
\label{fig:LOG_no_stra_loss}
\end{subfigure}
\hspace{-0.2cm}
\begin{subfigure}[t]{0.4\textwidth}
\centering
\includegraphics[width=\textwidth]{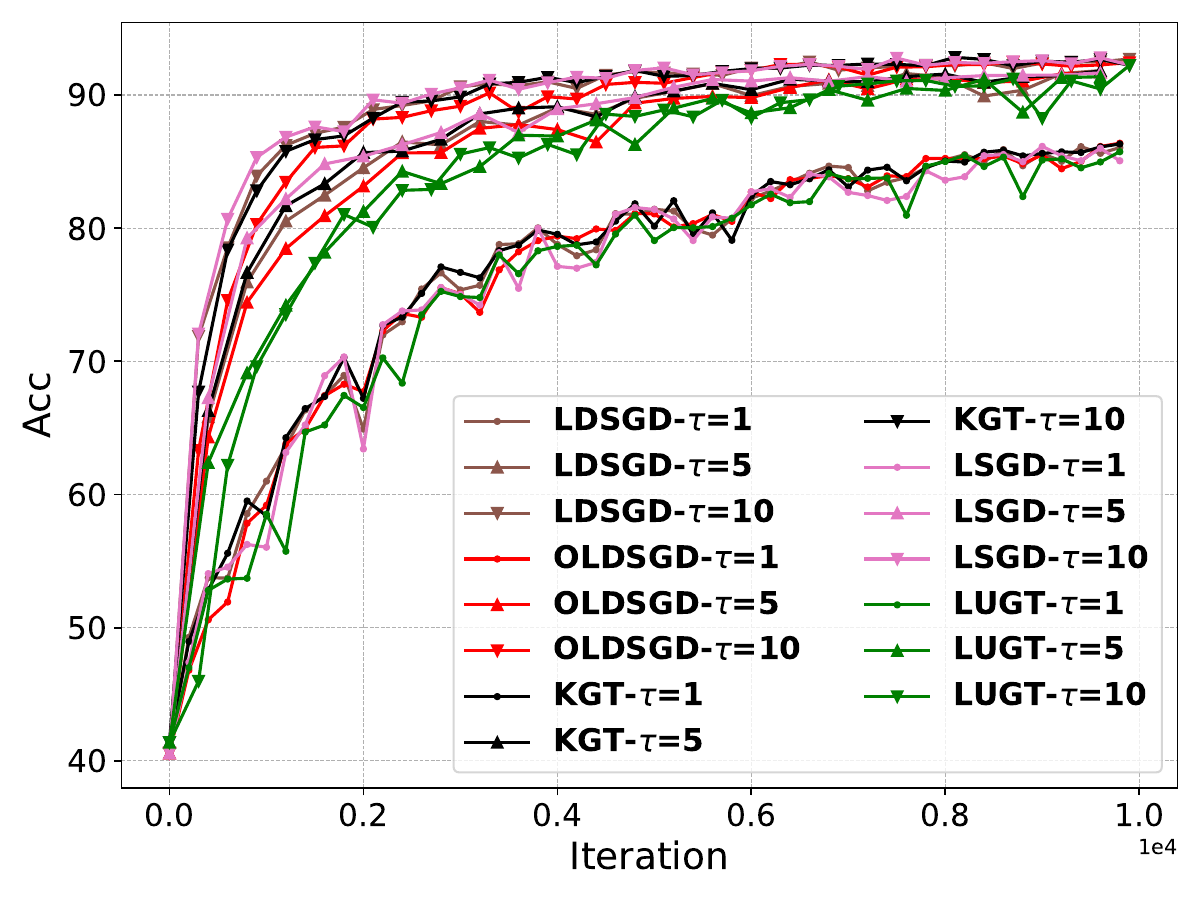}
\caption{Test Accuracy - Homo.}
\label{fig:LOG_stra_loss}
\end{subfigure}
\begin{subfigure}[t]{0.4\textwidth}
\centering
\includegraphics[width=\textwidth]{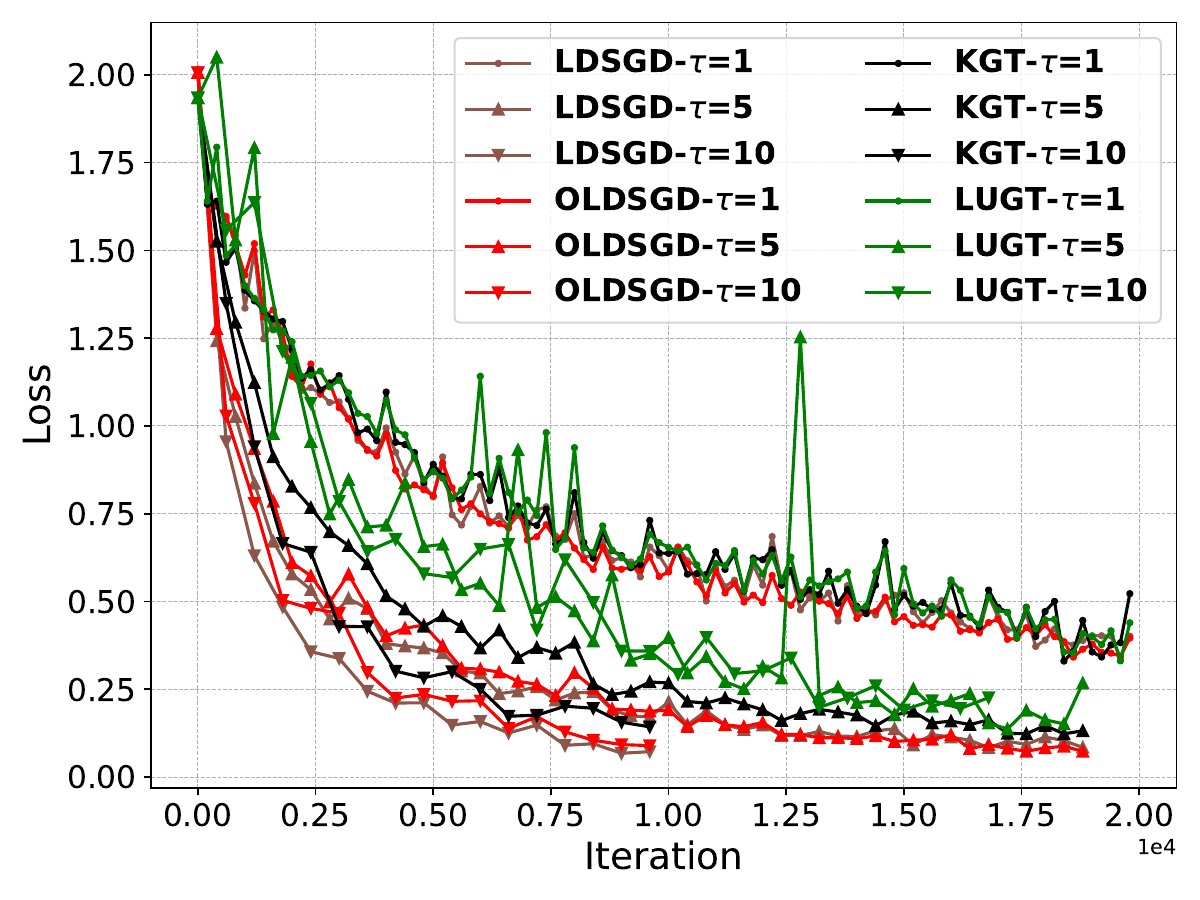}
\caption{Training Loss - Hete.} %
\label{fig:LOG_no_stra_loss}
\end{subfigure}
\hspace{-0.2cm}
\begin{subfigure}[t]{0.4\textwidth}
\centering
\includegraphics[width=\textwidth]{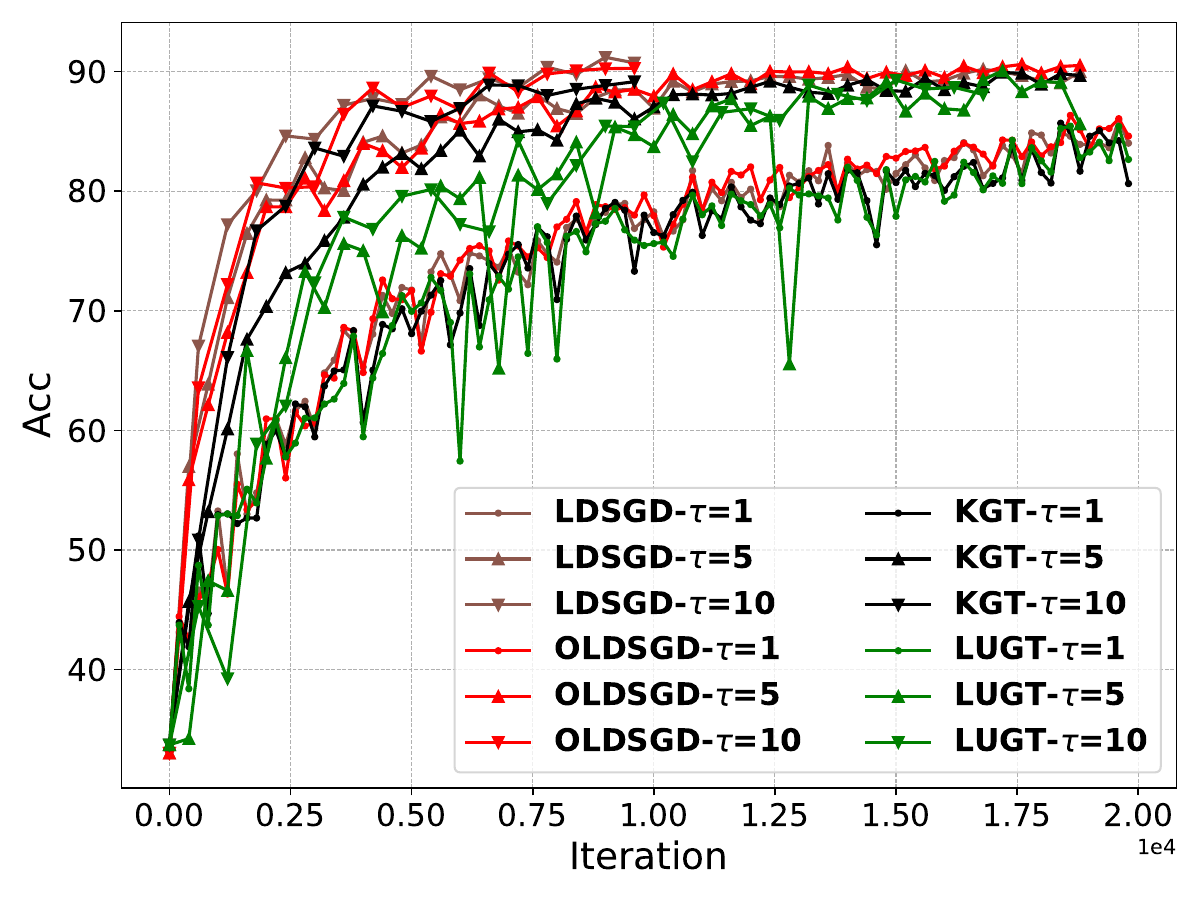}
\caption{Test Accuracy - Hete.}
\label{fig:LOG_stra_loss}
\end{subfigure}
\caption{Iteration-wise convergence of ResNet18 training. Different numbers of local steps are presented.}

\label{fig:ite_conv_resnet}
\end{figure*}

\begin{figure}
    \centering
    \includegraphics[width=0.7\linewidth]{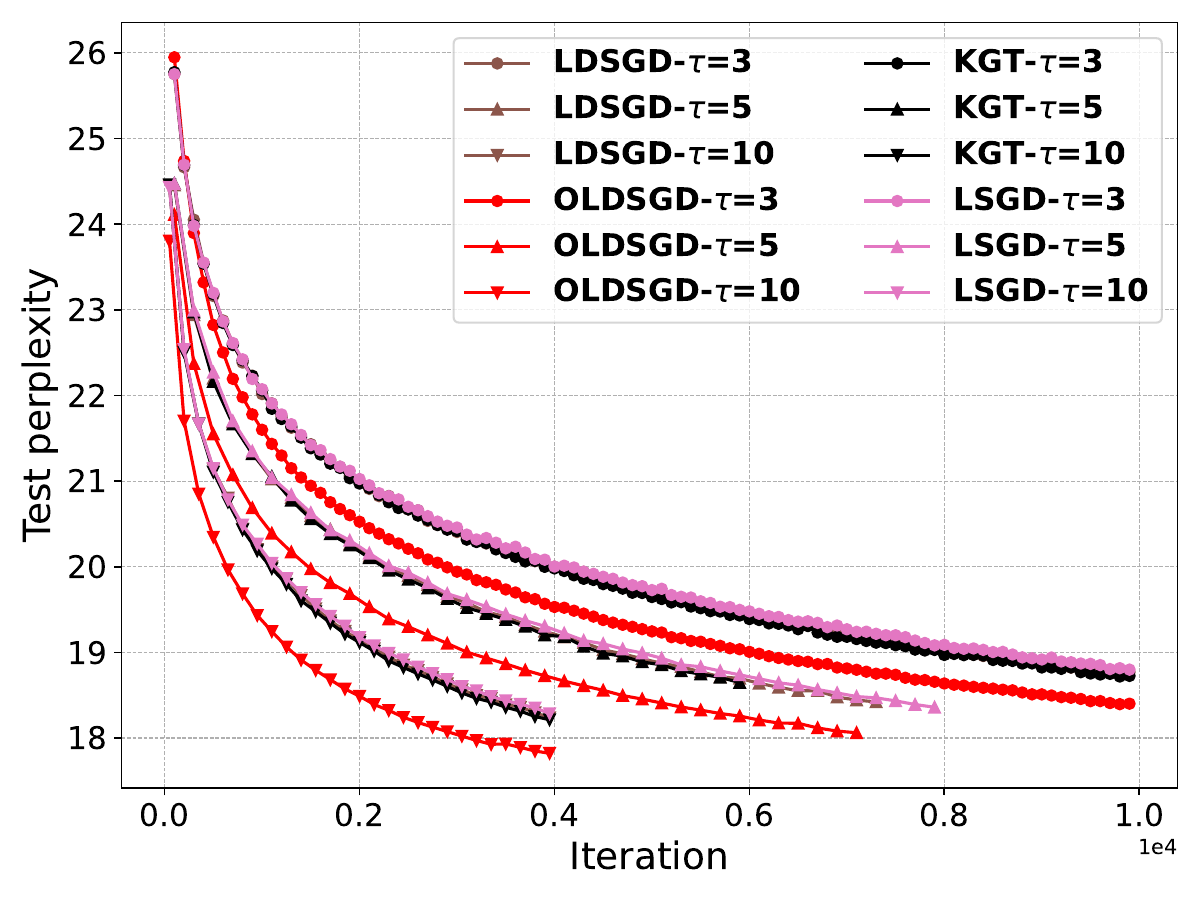}
    \caption{Test perplexity curves of GPT2 finetuning training w.r.t. iteration. OLDSGD constantly achieves lower perplexity than others.}
    \label{fig:ite_conv_gpt}
\end{figure}

\subsection{Speedup on heterogeneous cases}
In the main text, we present the speedup of OLDSGD in the homogeneous setting in Table \ref{tab:speedup}. As demonstrated in the previous subsection, OLDSGD tends to perform better (compared to other methods) under data heterogeneity. The following Table \ref{tab:speedu_hete} quantifies this trend. The last column represents the geometric means of speedup under the homogeneous case, calculated w.r.t. VGG and ResNet training only, for comparison. On average, OLDSGD achieves a 1.64$\times$ speedup compared to LDSGD under heterogeneous data conditions, surpassing the 1.39$\times$ speedup observed under homogeneous data. OLDSGD's speedup w.r.t. KGT and LUGT increases similarly under data heterogeneity. This enhanced performance under data heterogeneity underscores OLDSGD's robustness to variations in data distribution. Such a characteristic is advantageous for practical deployment, particularly in scenarios where homogeneous data is impractical due to constraints such as privacy concerns.

\begin{table}[htbp]
\centering
\caption{OLDSGD's Speedup Compared to Existing Methods (Higher Values Are Better). The last column presents geometric mean of speedup under the homogeneous case (excluding GPT2 finetuning) for comparison.}
\label{tab:speedu_hete}
\begin{tabular}{ccccccc}
\toprule
 & \multicolumn{2}{c}{\textbf{VGG11}} & \multicolumn{2}{c}{\textbf{ResNet18}} & \\
\cmidrule(lr){2-3} \cmidrule(lr){4-5}
\textbf{Algorithm} & $c=1$ & $c=5$ & $c=1$ & $c=5$ & \textbf{GeoMean} & \textbf{Homo. GeoMean}\\
\midrule
LDSGD & 1.41$\times$ & 1.56$\times$ & 1.54$\times$ & 1.71$\times$ & 1.64$\times$ & 1.39$\times$  \\
KGT & 1.34$\times$ & 2.25$\times$ & 2.63$\times$ & 2.25$\times$ & 1.63$\times$ & 1.42$\times$ \\
LUGT & 2.14$\times$ & 6.20$\times$ & 3.41$\times$ & 7.43$\times$ & 4.28$\times$ & 3.05$\times$ \\
\bottomrule
\end{tabular}
\end{table}



\end{document}